%% file: main.tex
\documentclass[letterpaper]{article} 
\usepackage{aaai24}  
\usepackage{times}  
\usepackage{helvet}  
\usepackage{courier}  
\usepackage[hyphens]{url}  
\usepackage{graphicx} 
\usepackage{natbib}  
\usepackage{caption} 
\usepackage{algorithm}
\usepackage{algorithmic}

\usepackage{newfloat}
\usepackage{listings}
\DeclareCaptionStyle{ruled}{labelfont=normalfont,labelsep=colon,strut=off} 
\floatstyle{ruled}
\newfloat{listing}{tb}{lst}{}
\floatname{listing}{Listing}
\newcommand{\norm}[1]{\left\lVert#1\right\rVert}
\newcommand{\br}[1]{\left(#1\right)}

\newcommand{\ones}{\mathbf{1}}
\newcommand{\zeros}{\mathbf{0}}

\usepackage{algorithm,algorithmic}

\input{import}
\usepackage{cleveref}
\theoremstyle{plain}
\newtheorem{theorem}{Theorem}

\newtheorem{lemma}[theorem]{Lemma}
\newtheorem{corollary}[theorem]{Corollary}
\theoremstyle{definition}
\newtheorem{definition}[theorem]{Definition}
\newtheorem{assumption}[theorem]{Assumption}
\theoremstyle{remark}
\newtheorem{remark}[theorem]{Remark}

\title{On Partial Optimal Transport: Revising the Infeasibility of Sinkhorn and Efficient Gradient Methods}
\author{
    Anh Duc Nguyen\textsuperscript{\rm 1},
    Tuan Dung Nguyen\textsuperscript{\rm 2},
     Quang Minh Nguyen\textsuperscript{\rm 3},
    Hoang H. Nguyen\textsuperscript{\rm 4}, \\
    Lam M. Nguyen\textsuperscript{\rm 5},
    Kim-Chuan Toh\textsuperscript{\rm 1, 6}
}
\affiliations{
    \textsuperscript{\rm 1}Department of Mathematics, National University of Singapore, Singapore\\
    \textsuperscript{\rm 2}Department of Computer and Information Science, University of Pennsylvania, USA\\
    \textsuperscript{\rm 3}Department of Electrical Engineering and Computer Science, Massachusetts Institute of Technology, USA\\
    \textsuperscript{\rm 4}School of Industrial and Systems Engineering, Georgia Institute of Technology, USA\\
    \textsuperscript{\rm 5}IBM Research, Thomas J. Watson Research Center, USA\\
    \textsuperscript{\rm 6}Institute of Operations Research and Analytics, National University of Singapore, Singapore\\

    anh$\_$duc@u.nus.edu, joshtn@seas.upenn.edu, nmquang@mit.edu, \\hnguyen455@gatech.edu,
    LamNguyen.MLTD@ibm.com, mattohkc@nus.edu.sg
}

\begin{document}
\maketitle
\begin{abstract}
This paper studies the Partial Optimal Transport (POT) problem between two unbalanced measures with at most $n$ supports and its applications in various AI tasks such as color transfer or domain adaptation. There is hence the need for fast approximations of POT with increasingly large problem sizes in arising applications. We first theoretically and experimentally investigate the infeasibility of the state-of-the-art Sinkhorn algorithm for POT due to its incompatible rounding procedure, which consequently degrades its qualitative performance in real world applications like point-cloud registration. To this end, we propose a novel rounding algorithm for POT, and then provide a feasible Sinkhorn procedure with a revised computation complexity of $\mathcal{\widetilde O}(n^2/\varepsilon^4)$. Our rounding algorithm also permits the development of two first-order methods to approximate the POT problem. The first algorithm, Adaptive Primal-Dual Accelerated Gradient Descent (APDAGD), finds an $\varepsilon$-approximate solution to the POT problem in $\mathcal{\widetilde O}(n^{2.5}/\varepsilon)$, which is better in $\varepsilon$ than revised Sinkhorn. The second method, Dual Extrapolation, achieves the computation complexity of $\mathcal{\widetilde O}(n^2/\varepsilon)$, thereby being the best in the literature. We further demonstrate the flexibility of POT compared to standard OT as well as the practicality of our algorithms on real applications where two marginal distributions are unbalanced.
\end{abstract}

\input{Main_Paper/introduction}
\input{Main_Paper/preliminaries}
\input{Main_Paper/revisiting_sinkhorn}
\input{Main_Paper/rounding_algorithm}
\input{Main_Paper/APDAGD}
\input{Main_Paper/DE}
\input{Main_Paper/experiments}
\input{Main_Paper/conclusion}

\bibliography{refs}
\onecolumn
\input{Appendix/appendix}
\input{Appendix/experiment_setup_details}
\input{Appendix/extra_experiments}

\end{document}

%% file: import.tex
\usepackage{microtype}
\usepackage{graphicx}
\usepackage{booktabs} 
\usepackage{amsmath}
\usepackage{amsfonts} 

\usepackage{epsf}
\usepackage{fancyhdr}
\usepackage{graphics}
\usepackage{graphicx}
\usepackage{psfrag}
\usepackage{microtype}
\usepackage{mathtools}

\usepackage{color}
\usepackage{amsthm}
\usepackage{amsfonts}
\usepackage{amsmath}
\usepackage{amssymb}
\usepackage{caption}
\usepackage{sidecap}
\sidecaptionvpos{figure}{c}

\newcommand{\diag}{\textnormal{diag}}

\long\def\comment#1{}
\comment{
\setlength{\topmargin}{0 in}
\setlength{\textwidth}{5.5 in}
\setlength{\textheight}{8 in}
\setlength{\oddsidemargin}{0.5 in}
}

\newtheorem{observation}{Observation}

\def\argmin{\textnormal{arg} \min}

\newcommand{\inner}[2]{\langle#1, #2\rangle}

\newcommand{\one}{\textbf{1}}

\def\RR{\mathbb{R}}

\def\xX{\mathbf{x}}

\graphicspath{{./image/}}

\newcommand{\argmax}{\mathop{\arg \max}}

\newcommand{\ba}{\begin{array}}
\newcommand{\ea}{\end{array}}

\newcommand{\defeq}{\vcentcolon=}

\newcommand{\va}{\mathbf{a}}
\newcommand{\vb}{\mathbf{b}}
\newcommand{\vc}{\mathbf{c}}
\newcommand{\vd}{\mathbf{d}}
\newcommand{\ve}{\mathbf{e}}

\newcommand{\vg}{\mathbf{g}}
\newcommand{\vh}{\mathbf{h}}

\newcommand{\vp}{\mathbf{p}}
\newcommand{\vq}{\mathbf{q}}
\newcommand{\vr}{\mathbf{r}}
\newcommand{\vs}{\mathbf{s}}
\newcommand{\vt}{\mathbf{t}}
\newcommand{\vu}{\mathbf{u}}
\newcommand{\vv}{\mathbf{v}}
\newcommand{\vw}{\mathbf{w}}
\newcommand{\vx}{\mathbf{x}}
\newcommand{\vy}{\mathbf{y}}
\newcommand{\vz}{\mathbf{z}}

\newcommand{\vA}{\mathbf{A}}

\newcommand{\vC}{\mathbf{C}}

\newcommand{\vI}{\mathbf{I}}

\newcommand{\vR}{\mathbf{R}}
\newcommand{\vS}{\mathbf{S}}
\newcommand{\vT}{\mathbf{T}}

\newcommand{\vX}{\mathbf{X}}
\newcommand{\vY}{\mathbf{Y}}

\newcommand{\vecflatten}{\text{vec}}

%% file: Main_Paper/introduction.tex
\section{Introduction}
Optimal Transport (OT) \cite{Villani-09-Optimal, Kantorovich-1942-Translocation}, which seeks a minimum-cost coupling between two balanced measures, is a well-studied topic in mathematics and operations research. With the introduction of entropic regularization  \cite{Cuturi-2013-Sinkhorn}, the scalability and speed of OT computation have been significantly improved, facilitating its widespread applications in machine learning such as domain adaptation \citep{Courty-2017-Optimal}, and dictionary learning \citep{pmlr-v51-rolet16}. 
However, OT has a stringent requirement that the input measures must have equal total masses  \citep{Chizat_2015}, hindering its practicality in various other machine learning applications, which require an optimal matching between two measures with unbalanced masses, such as averaging of neuroimaging data \citep{Gramfort_2015} and image classification \citep{Pele2008ALT, rubner2000earth}.

In response to such limitations, Partial Optimal Transport (POT), which explicitly constrains the mass to be transported between two unbalanced measures, was proposed. It has been studied from the perspective of partial differential equations by theorists \cite{figalli2010optimal, Caffarelli_2010}. Practically, the relaxation of the marginal constraints, which are strictly imposed by the standard OT, and the control over the total mass transported grant POT immense flexibility compared to OT \citep{Chapel-nips2020} and more robustness to outliers \citep{nhatho-mmpot}. POT has been deployed in various recent AI applications such as color transfer \citep{Bonneel_2019}, graph neural networks \citep{Sarlin_2019}, graph matching \cite{Liu_2020},  partial covering \citep{Kawano_2021}, point set registration \citep{Wang_2022}, and robust estimation \citep{nietert2023robust}.   


Despite its potential applicability, POT still suffers from the computation bottleneck, whereby the more intricate structural constraints imposed on admissible couplings have hindered the direct adaptation of any efficient OT solver in the literature. Currently, the literature \citep{Chapel-nips2020, nhatho-mmpot} relies on reformulating POT into an extended OT problem under additional assumptions on the input masses, which can then be solved via existing OT methods, and finally retrieves an admissible POT coupling from the solution to the extended OT problem. This approach has two fundamental drawbacks. First, in the reformulated OT problem, the maximum entry of the extended cost matrix is increased \citep[Proposition 1]{Chapel-nips2020}, which will always worsen the computational complexity since most efficient algorithms for standard OT \citep{Dvurechensky-2018-Computational, lin2019efficient, guminov2021accelerated} depend on this maximum entry in their complexities. 
Second, we discover, more details in the Revisiting Sinkhorn section, that although Sinkhorn for POT proposed by \citep{nhatho-mmpot} achieves the best known complexity of $\mathcal{\widetilde O}(n^2/\varepsilon^2)$, it in fact always outputs a \emph{strictly infeasible} solution to the POT problem. In brief, by discarding the last row and column of the reformulated OT solution obtained by Sinkhorn, the POT solution, transportation matrix $\vX$, will violate the equality constraint $\ones^\top \vX \ones = s$, which controls the total transported mass. Violating this equality constraint can degrade the results of practical applications in robust regimes such as point cloud registration \citep{qin2022rigid} and mini-batch OT \citep{nguyen2022improving} (refer to Remark \ref{remark:violation} and our Point Cloud Registration experiment). We theoretically justify the ungroundedness of Sinkhorn for POT in the Revisiting Sinkhorn section and empirically verify this claim in several applications in the Numerical Experiment section. Here, in Figure \ref{fig:Infeasible_Sinkhorn}, we specifically investigate Sinkhorn infeasibility in a color transfer example (detailed experimental setup in the later Numerical Experiment section). We show that the optimality gap produced by Sinkhorn is unable to reduce lower than $\varepsilon$, the tolerance of the problem (red line). In other words, Sinkhorn fails to produce an adequate $\varepsilon$-approximate POT solution. 

\begin{table*}[t]
    \centering
    \renewcommand\arraystretch{1.12}
    \begin{tabular}{c c c c}
        \toprule
        Algorithm & Regularizer & Cost per iteration & Iteration complexity \\
        \hline
        Iterative Bregman projections \cite{benamou2015iterative} & Entropic & Unspecified & Unspecified\\ 
        (Infeasible) Sinkhorn \citep{nhatho-mmpot} & Entropic & $\mathcal{O}(n^2)$ & $\widetilde{\mathcal{O}}(1/\varepsilon^2)$\\
        (Feasible) Sinkhorn (\textbf{This paper}) & Entropic & $\mathcal{O}\br{n^2}$ & $\widetilde{\mathcal{O}}\br{1/\varepsilon^4}$ \\
        APDAGD (\textbf{This paper}) & Entropic & $\mathcal{O}(n^2)$ & $\widetilde{\mathcal{O}}\br{\sqrt{n}/\varepsilon}$ \\
        Dual Extrapolation (\textbf{This paper}) & Area-convex & $\mathcal{O}(n^2)$ & $\widetilde{\mathcal{O}}\br{1/\varepsilon}$\\
        \bottomrule
        \end{tabular}
    \caption{Type of regularizers and orders of complexity for four algorithms for POT approximation.} \label{Order_of_complexities}
\end{table*}

\begin{figure}
    \centering
    \includegraphics[width=0.8\linewidth]{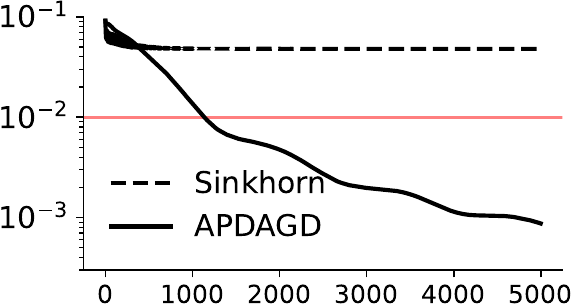}
    \caption{Primal optimality gap against optimization rounds achieved by Sinkhorn \cite{nhatho-mmpot} and APDAGD for POT. The marginal distributions are taken from a color transfer application described in our Numerical Experiments section, and the red horizontal line depicts the pre-defined tolerance ($\varepsilon$) for both algorithms.}
    \label{fig:Infeasible_Sinkhorn}
\end{figure}

To the best of our knowledge, the invalidity of Sinkhorn means there is currently no efficient method for solving POT in the literature. We attribute this to the fact that while the equivalence between POT and extended OT holds at optimality, all efficient OT solvers instead only output an approximation of the optimum value before projecting it back to the feasible set. However, the well-known rounding algorithm by \citep{altschuler2017near}, which is specifically designed for OT, does not guarantee to respect the more intricate structural constraints of POT, resulting in the invalidity of Sinkhorn. Motivated by this challenge and the success of optimization literature for OT, we raise the following central question of this paper: 
\begin{center}
\emph{Can we design a rounding algorithm for POT and then utilize it to develop efficient algorithms for POT that even match the best-known complexities of those for OT?} 
\end{center}
We affirmatively answer this question and formally summarize our contributions as follows.

\begin{itemize}
    \item We theoretically and experimentally show the infeasibility of the state-of-the-art Sinkhorn algorithm for POT  due to its incompatible rounding algorithm. We propose a novel POT rounding procedure \textsc{Round-POT} (Rounding Algorithm Section), which projects an approximate solution onto the feasible POT set in $\mathcal{O}(n^2)$ time. 
    \item From our theoretical bounds of the Sinkhorn constraint violations and the newly introduced \textsc{Round-POT}, we provide a revised procedure for Sinkhorn which will return a feasible POT solution. We also establish the revised complexity of Sinkhorn for POT (Table \ref{Order_of_complexities}).
    \item Predicated on our novel dual formulation for entropic regularized POT objective, our proposed Adaptive Primal-Dual Accelerated Gradient Descent (APDAGD) algorithm for POT finds an $\varepsilon$-approximate solution in  $\mathcal{\widetilde O}(n^{2.5}/\varepsilon)$, which is better in $\varepsilon$ than the revised Sinkhorn.  
    Various experiments on synthetic and real datasets and with applications such as point cloud registration, color transfer, and domain adaptation illustrate not only our algorithms' favorable performance against the pre-revised Sinkhorn but also the versatility of POT compared to OT. 
    \item Motivated by our novel rounding algorithm, we further reformulate the POT problem with $\ell_1$ penalization as a minimax problem and propose Dual Extrapolation (DE) framework for POT. We prove that DE algorithm can theoretically achieve $\mathcal{\widetilde O}(n^{2}/\varepsilon)$ computational complexity, thereby being the best in the POT literature to the best of our knowledge (Table \ref{Order_of_complexities}). 
\end{itemize}

%% file: Main_Paper/preliminaries.tex
\section{Preliminaries}

\subsection{Notation} 

The set of non-negative real numbers is $\RR_+$. We use bold capital font for matrices (e.g., $\vA$) and bold lowercase font for vectors (e.g., $\vx$). For an $m \times n$ matrix $\vX$, $\vecflatten(\vX)$ denotes the $(m n)$-dimensional vector obtained by concatenating the rows of $\vX$ and transposing the result. Entrywise multiplication and division for matrices and vectors are respectively denoted by $\odot$ and $\oslash$. For $1 \leq p \leq \infty$, let $\norm{\cdot}_p$ be the $\ell_p$-norm of matrix or vector. For matrices, $\norm{\cdot}_{p \rightarrow q}$ is the operator norm: $\norm{\vA}_{p \rightarrow q} = \sup_{\norm{\vx}_p = 1} \norm{\vA \vx}_q$. Three specific cases are considered in this paper: for $q \in \{1, 2, \infty\}$, $\norm{\vA}_{1 \rightarrow q}$ is the largest $\ell_q$ norm of any column of $\vA$. We use $\norm{\vA}_{\max}$ and $\norm{\vA}_{\min}$ to denote the maximum and minimum entries in absolute value of a matrix $\vA$, respectively. The $n$-vectors of zeros and of ones are respectively denoted by $\zeros_n$ and $\ones_n$. The $(n-1)$-dimensional probability simplex is $\Delta_n = \left\{ \vv \in \RR_+^n : \vv^\top \ones_n = 1 \right\}$. 

\subsection{Partial Optimal Transport} 




Consider two discrete distributions $\vr, \vc \in \RR_+^{n}$ with possibly different masses. POT seeks a transport plan $\vX \in \RR_{+}^{n \times n}$ which maps $\vr$ to $\vc$ at the lowest cost. Since the masses at two marginals may differ, only a total mass $s$ such that $0 \leq s \leq \min\{\norm{\vr}_1, \norm{\vc}_1\}$ is allowed to be transported \citep{Chapel-nips2020, nhatho-mmpot}. Formally, the POT problem is written as
\begin{align} \label{eq:pot_formulation}
  \mathbf{POT}(\vr, \vc, s) = \min \inner{\vC}{\vX} ~~\text{s.t.}~~ \vX \in \mathcal{U}(\vr, \vc, s),
\end{align}
where $\mathcal{U}(\vr, \vc, s)$ is defined as $$ \left\{ \mathbf{X} \in \RR_{+}^{n \times n} : \mathbf{X} \ones_n \leq \vr, \mathbf{X}^\top \ones_n \leq \vc, \ones_n^\top \mathbf{X} \ones_n = s \right\},$$ 
i.e. the feasible set for the transport map $\vX$ is and $\vC \in \RR^{n \times n}_+$ is a cost matrix.
The goal of this paper is to derive efficient algorithms to find an $\varepsilon$-approximate solution to $\mathbf{POT}(\vr, \vc, s)$, pursuant to the following definition.

\begin{definition}[$\varepsilon$-approximation]
For $\varepsilon \geq 0$, the matrix $\vX \in \RR_{+}^{n \times n} $ is an $\varepsilon$-approximate solution to $ \mathbf{POT}(\vr, \vc, s)$ if $\vX \in \mathcal{U}(\vr, \vc, s)$ and
\begin{equation*}
    \inner{\vC}{{\vX}} \leq \min \inner{\vC}{\vX'} + \varepsilon ~~ \text{s.t.} ~~ \vX' \in \mathcal{U}(\vr, \vc, s).
\end{equation*}
\end{definition}

To aid the algorithmic design in the following sections, we introduce two new slack variables $\vp, \vq \in \RR_+^{n}$ and equivalently express problem \eqref{eq:pot_formulation} as
\begin{align}
    \label{prob:pot_with_pq}
        &\min_{\vX \geq 0, \vp \geq 0, \vq \geq 0} \inner{\vC}{\vX} \\
        &\text{s.t. } \: \mathbf{X} \ones_n + \vp = \vr, \mathbf{X}^\top \ones_n + \vq = \vc, \ones_n^\top \mathbf{X} \ones_n = s.
\end{align}
We also study a convenient equivalent formulation of this problem:
\begin{equation} \label{main_problem}
    \min_{\vx \geq \zeros} \:  \inner{\vd}{\vx} ~~ \text{s.t.} ~~ \vA\vx = \vb,
\end{equation}
where we perform vectorization with $\vd^\top = (\vecflatten(\vC)^\top, \zeros_{2n}^\top)$ and  $\vx^\top = (\vecflatten(\vX)^\top, \vp^\top, \vq^\top)$. 
The  constraints in Equation \ref{prob:pot_with_pq} are encoded in $\vA\vx = \vb$, where $\vA \in \RR^{(2n + 1) \times (n^2+2n)}$ and $\vb \in H \defeq \RR^{2n+1}$ such that $(\vA \vx)^\top = ((\vX \ones + \vp)^\top, (\vX^\top \ones + \vq)^\top, \ones^\top \vX \ones)$ and $\vb^\top = (\vr^\top, \vc^\top, s)$. In other words, the linear operator $\vA$ has the form
\begin{equation*} \label{matrix_A}
    \vA = \left(
    \begin{array}{cc}
    \vA' & \vI_{2n} \\
    \ones_{n^2}^\top & \zeros_{2n}
    \end{array} \right),
\end{equation*}
where $\vA'$ is the edge-incidence matrix of the underlying bipartite graph in OT problems \citep{Dvurechensky-2018-Computational,Jambulapati-2019-Direct}.


%% file: Main_Paper/revisiting_sinkhorn.tex
\section{Revisiting Sinkhorn for POT} 
We can reformulate Problem \eqref{eq:pot_formulation} by adding \emph{dummy points} and extending the cost matrix as
\begin{align*}
    \widetilde{\vC} = \left(
    \begin{array}{cc}
    \vC & \zeros_{n} \\
    \zeros_{n}^\top & A
    \end{array} \right) \in \RR_{+}^{(n+1) \times (n+1)},
\end{align*}
where $A > \max(C_{i, j})$ \citep{Chapel-nips2020}. Then the two marginals are augmented to $(n+1)$-dimensional vectors as $\widetilde{\vr}^\top = (\vr^\top, \norm{\vc}_1 - s)$ and $\widetilde{\vc}^\top = (\vc^\top, \norm{\vr}_1 - s)$. \citep[Proposition 1]{Chapel-nips2020} show that one can obtain the solution POT by solving this extended OT problem with balanced marginals $\widetilde{\vr}, \widetilde{\vc}$ and cost matrix $\widetilde{\vC}$. In particular, if the OT problem admits an optimal solution of the form
\begin{align*}
    \widetilde{\vX} = \left(
    \begin{array}{cc}
    \bar{\vX} & \widetilde{\vp} \\
    \widetilde{\vq}^\top & \widetilde{X}_{n+1,n+1}
    \end{array} \right) \in \RR_{+}^{(n+1) \times (n+1)},
\end{align*}
then $\bar{\vX} \in \RR_{+}^{n \times n}$ is the solution to the original POT.
\citep{nhatho-mmpot} seeks an approximate solution to the extended OT problem using the Sinkhorn algorithm (see Algorithm \ref{alg:Sinkhorn} in the Appendix). Then the rounding procedure by \citep{altschuler2017near} is applied to the solution to give a primal feasible matrix. While the two POT inequality constraints are satisfied, we discover in the following Theorem that the equality constraint $\ones^\top \Bar{\vX} \ones = s$ is violated. The proof is in Appendix, Revisiting Sinkhorn for POT section.
\begin{theorem}
    \label{contraint_violation}
    For a POT solution $\Bar{\vX}$ from \citep{nhatho-mmpot}, the constraint violation $V \defeq \mathbf{1}^\top \bar{\vX} \mathbf{1} -s$ can be bounded as
    \begin{equation*}
    \tilde{\mathcal{O}}\left( \frac{\|\vC^2\|_{\text{max}}}{A}\right) \geq V \geq \exp\left(\frac{-12A \log{n}}{\varepsilon} - \mathcal{O}(\log n)\right).
    \end{equation*}
\end{theorem}
\textbf{Feasible Sinkhorn Procedure:} With these bounds, we deduce that in order for Sinkhorn to be feasible, one needs to \textbf{both} utilize our \textsc{Round-POT} and choose a sufficiently large $A$ (Theorem \ref{them:revised_sinkhorn_complexity}) as opposed to the common practice of picking $A$ a bit larger than 1 \citep{nhatho-mmpot}. We derive the revised complexity of Sinkhorn for POT in the following theorem.
\begin{theorem}
    \label{them:revised_sinkhorn_complexity}
    (Revised Complexity for Feasible Sinkhorn with \textsc{Round-POT}) We first derive the sufficient size of $A$ to be $\mathcal{O} \left( \frac{\|\mathbf{C}\|_{\text{max}}}{\varepsilon} \right)$. With this large $A$ and \textsc{Round-POT}, Sinkhorn for POT has a computational complexity of $\tilde{\mathcal{O}} \left(\frac{n^2 \|\mathbf{C}\|^2_{\text{max}} }{\varepsilon^4} \right)$ as oppose to $\tilde{\mathcal{O}} \left(\frac{n^2 \|\mathbf{C}\|^2_{\text{max}} }{\varepsilon^2} \right)$ \cite{nhatho-mmpot}. 
\end{theorem}
The detailed proof for this theorem is included in Appendix, Revisiting Sinkhorn for POT section. We also empirically verify this worsened complexity in section in Feasible Sinkhorn section in Appendix. 
\begin{remark} \label{remark:violation}
Respecting the equality constraint is crucial for various applications that demand strict adherence to feasible solutions such as point cloud registration \citep{qin2022rigid} (for avoiding  incorrect many-to-many correspondences) and mini-batch OT \citep{nguyen2022improving} (for minimizing misspecification). 
Consequently, it is imperative for POT to transport the exact fraction of mass to achieve an optimal mapping, which is vital for the effective performances of ML models.
\end{remark}


%% file: Main_Paper/rounding_algorithm.tex
\begin{algorithm}[H]
    \caption{\textsc{Round-POT}}
    \label{alg:rounding}
    \begin{algorithmic} [1]
        \REQUIRE{ $\vx = (\vecflatten(\vX)^\top, \vp^\top, \vq^\top)^\top$; marginals $\vr$, $\vc$; mass $s$.}
        \STATE \( \bar{\vp} = \mathtt{EP}(\vr, s, \vp) \) 
        \STATE \( \bar{\vq} = \mathtt{EP}(\vc, s, \vq) \)
        \STATE \(\vg = \min\{\ones, (\vr-\bar{\vp}) \oslash \vX \ones \}\)
        \STATE\(\vh = \min\{\ones, (\vc-\bar{\vq}) \oslash \vX^\top \ones \} \)
        \STATE \( \vX' = \diag(\vg) \vX \diag(\vh)\)
        \STATE \( \ve_1 = (\vr - \bar{\vp}) - \vX' \one, \ve_2 = (\vc - \bar{\vq}) - \vX'^\top \one\)
        \STATE \( \bar{\vX} = \vX' + \ve_1 \ve_2^\top / \norm{\ve_1}_1 \)
        \ENSURE{$ \Bar{\vx} = (\Bar{\vX}$, $\Bar{\vp}$,  $\Bar{\vq})$}
    \end{algorithmic}
\end{algorithm}
\section{Rounding Algorithm}
All efficient algorithms for standard OT
\citep{Dvurechensky-2018-Computational, lin2019efficient, guminov2021accelerated}
only output an infeasible approximation of the optimum value, and leverage the well-known rounding algorithm \citep[Algorithm 2]{altschuler2017near} to project it back to the set of admissible couplings. Nevertheless, its ad-hoc design tailored to the OT's marginal constraints makes  generalization to the case of POT with more intricate structural constraints non-trivial. In fact, we attribute the rather limited literature on efficient POT solvers to such lack of a rounding algorithm for POT. Specifically, previous works rely on imposing additional assumptions on the input masses to permit reformulation of POT into standard OT with an additional computational burden \cite{Chapel-nips2020, nhatho-mmpot}. Deviating from the vast literature, we address this fundamental challenge by proposing a novel rounding procedure for POT, termed \textsc{Round-POT} (Algorithm \ref{alg:rounding}), to efficiently round any approximate solution to a feasible solution of \eqref{prob:pot_with_pq}.
Given an approximate solution $\xX = (\vX, \vp, \vq) \geq 0$ violating the POT constraints of \eqref{prob:pot_with_pq} by a predefined error, $\norm{\vA \vx - \vb}_1 \leq \delta$ for some $\delta$,  \textsc{Round-POT} returns $\Bar{\vx} = (\Bar{\vX}, \Bar{\vp}, \Bar{\vq}) \geq \zeros$ strictly in the feasible set, i.e., $\vA \Bar{\vx} = \vb$, and close to $\vx$ in $\ell_1$ distance. 

The Enforcing Procedure ($\mathtt{EP}$) (Algorithm \ref{alg:EP}) is a novel subroutine to ensure $\zeros \leq \bar{\vp} \leq \vr$ and $\norm{\bar{\vp}}_1 = \norm{\vr}_1 - s$ (Lemma \ref{Guarantees_for_EP}). Equivalently, a similar procedure is applied to ($\vc, s, \vq$) in step 2 of Algorithm \ref{alg:rounding} with similar guarantees for $\bar{\vq}$. Step 1 transforms $\vp$ (or $\vq$) to $\vp'$ (or $\vq'$) so that $\zeros \leq \vp' \leq \vr$ (or $\zeros \leq \vq' \leq \vc$). The transformation in steps 2 and 3 ensures that $\norm{\vp''}_1 \leq \norm{\vr}_1 - s$ (or $\norm{\vq''}_1 \leq \norm{\vc}_1 - s$). The rest of the $\mathtt{EP}$ steps ensure the other guarantee $\norm{\bar{\vp}}_1 = \norm{\vr}_1 - s$.
The proof is included in Rounding Algorithm section in the Appendix. 
\begin{lemma}
    \label{Guarantees_for_EP}
    (Guarantees for $\mathtt{EP}$) We obtain in $\mathcal{O}(n)$ time $\zeros \leq \bar{\vp} \leq \vr$ and $\norm{\bar{\vp}}_1 = \norm{\vr}_1 - s$. 
\end{lemma}
For \textsc{Round-POT}, steps 3 through 7 check whether the solutions $\vX$ violate each of the two equality constraints $\vX \ones = \vr - \bar{\vp}$ and $\vX^\top \ones = \vc - \bar{\vq}$; if so, the algorithm projects $\vX$ into the feasible set. It is noteworthy that these two constraints directly implies the last needed constraint $\ones^\top \vX \ones = s$. Finally, \textsc{round-POT} returns an output that satisfies the required constraints in Equation \eqref{prob:pot_with_pq}.
The following Theorem \ref{prop:rounding} characterizes the error guarantee of the rounded output $\Bar{\vx}$. Its detailed proof can be found in Rounding Algorithm section in the Appendix.

\begin{algorithm}[H]
            \caption{Enforcing Procedure $\mathtt{EP}$}
            \label{alg:EP}
            \begin{algorithmic}[1]
                \REQUIRE{marginal $\vr$; mass $s$; slack variable $\vp$.}
                \STATE \( \vp' = \min\{\vp, \vr\}, \alpha = \min\left\{ 1, (\norm{\vr}_1 - s) / \norm{\vp'}_1 \right\} \)
                \STATE \( \vp'' = \alpha \vp' \)
                \IF{$1 > (\norm{\vr}_1-s) / \norm{\vp'}_1$}
                    \STATE \(\bar{\vp} = \vp''\)
                \ELSE 
                    \STATE{$i = 0$}
                    \WHILE{$\norm{\vp''}_1 \leq \norm{\vr}_1 -s$}
                        \STATE $i = i + 1$ \\
                        \STATE \( p''_i = r_i \)
                    \ENDWHILE \\
                    \STATE \( p''_i = p''_i - (\norm{\vp''}_1 - \norm{\vr}_1 +s)\) 
                    \STATE $\bar{\vp} = \vp''$
                \ENDIF 
                \ENSURE{$\bar{\vp}$}.
            \end{algorithmic}
\end{algorithm}

\begin{theorem}
    (Guarantees for $\textsc{Round-POT}$)
    \label{prop:rounding}
    Let $\vA$, $\vx$ (consisting of $\vX$, $\vp$ and $\vq$) and $\vb$ be defined as in the preliminaries. If $\vx$ satisfies that $\vx \geq 0$ and $\norm{\vA \vx - \vb}_1 \leq \delta$ for some $\delta \geq 0$, Algorithm \ref{alg:rounding} outputs $\Bar{\vx} \geq \zeros$ (consisting of $\Bar{\vX}, \Bar{\vp}$ and $\Bar{\vq}$) in $\mathcal{O}(n^2)$ time such that $\vA \Bar{\vx} = \vb$ and $\norm{\vx - \Bar{\vx}}_1 \leq 23 \delta$.
\end{theorem}

%% file: Main_Paper/APDAGD.tex
\section{Adaptive Primal-Dual Accelerated Gradient Descent (APDAGD)}

\subsection{Dual Formulation and Algorithmic Design}
Following a similar formulation to \citep[Section 3.1]{Dvurechensky-2018-Computational}, we have the following primal problem with entropic regularization
\begin{align}
    \label{prob:entropic}
    \min_{\vx \geq \zeros} \left\{ f(\vx) \defeq \inner{\vd}{\vx} + \gamma \inner{\vx }{\log \vx} \right\} ~~ \text{s.t.} ~~  \vA \vx = \vb,
\end{align}
where $\vA \vx = \vb$ is encoded as explained in Equation \eqref{main_problem}. Since problem \eqref{prob:entropic} is a linearly constrained convex optimization problem, strong duality holds.
\begin{lemma}
    With a dual variable $\pmb{\lambda} \in H^{*} = \RR^{2n+1}$, the dual of \eqref{prob:entropic} is given by
        $$\min_{\pmb{\lambda} \in H^{*}} \left\{ \varphi(\pmb{\lambda}) \defeq \inner{\pmb{\lambda}}{\vb} + \max_{\vx \in Q} \left\{ - f(\vx) - \inner{\vx}{\vA^\top \pmb{\lambda}}\right\} \right\},$$
    or equivalently
    \begin{align}
    \label{entropic_regularized}
        \begin{split}
        &\min_{\vy, \vz, t} \left\{- ts - \inner{\vy}{\vr} - \inner{\vz}{\vc}  \right. \\
        &\left. - \gamma \sum_{i, j=1}^{n} e^{-(C_{i, j} + y_i + z_j + t)/{\gamma}-1} + e^{-y_i / \gamma - 1} +  e^{-z_j/ \gamma - 1}  \right\},
        \end{split}
\end{align}
     where $\vy, \vz, t$ are dual variables corresponding the POT constraints in \eqref{prob:pot_with_pq} as $\pmb{\lambda} = (\vy^\top, \vz^\top, t)^\top$ (which we simply refer as $(\vy, \vz, t)$ from now on).
\end{lemma}
The detailed dual formulation of Equation \eqref{prob:entropic} is found in subsection Dual Formulation for Entropic POT in the Appendix.

More details on the properties (strong convexity, smoothness, etc) of the primal and dual objectives are in Properties of Entropic POT section in Appendix. The APDAGD procedure is described in Algorithm \ref{alg:APDAGD} in Appendix. So as to approximate POT, we incorporate our novel rounding algorithm with a similar procedure to \citep[Algorithm 2]{lin2019efficient}, in Algorithm \ref{alg:ApproxOT_APDAGD}.
\begin{algorithm}
\caption{Approximating POT by APDAGD} \label{alg:ApproxOT_APDAGD}
\begin{algorithmic}[1]
    \REQUIRE{ marginals $\vr, \vc$; cost matrix $\vC$}.
    \STATE $\gamma = \varepsilon / (4\log(n)), \widetilde{\varepsilon}= \varepsilon / (8\|\vC\|_{\max})$
    \IF{$\| \vr \|_1 > 1$}
        \STATE \(\widetilde{\varepsilon} = \min\left\{\widetilde{\varepsilon}, 8 (\|\vr\|_1-s) / (\|\vr\|_1-1) \right\} \)
    \ENDIF
    \IF{$\| \vc \|_1 > 1$}
        \STATE \(\widetilde{\varepsilon} = \min\left\{\widetilde{\varepsilon}, 8 (\|\vc\|_1-s) / (\|\vc\|_1-1) \right\} \)
    \ENDIF
    \STATE $\widetilde{\vr} = \left(1 - \widetilde{\varepsilon}/ 8\right)\vr + \widetilde{\varepsilon} \ones_n / (8n)$
    \STATE $\widetilde{\vc} = \left(1 - \widetilde{\varepsilon}/ 8\right)\vr + \widetilde{\varepsilon} \ones_n / (8n)$
    \STATE $\widetilde{\vX} = \textsc{Apdagd}(\vC, \gamma, \widetilde{\vr}, \widetilde{\vc}, \widetilde{\varepsilon}/2)$
    \STATE $\bar{\vX} = \textsc{Round-POT}(\widetilde{\vX}, \widetilde{\vr}, \widetilde{\vc}, s)$
    \ENSURE{$\bar{\vX}$}.  
\end{algorithmic}
\end{algorithm}
\subsection{Computational Complexity}
Now, we provide the computational complexity of APDAGD (Theorem \ref{APDAGD_complexity}) and its proof sketch. The detailed proof of this result will be presented in Complexity of APDAGD for POT Detailed Proof subsection in Appendix.
\begin{theorem} \label{APDAGD_complexity} (Complexity of APDAGD)
    The APDAGD algorithm returns $\varepsilon$-approximation POT solution $\widehat{\vX} \in \mathcal{U}(\vr, \vc, s)$ in $\mathcal{\widetilde{O}}\left(n^{5/2} \norm{\vC}_{\max} / \varepsilon \right)$. 
\end{theorem}
\begin{proof}[Proof sketch] \textbf{Step 1}: we present the reparameterization $\vu = - \vy / \gamma - \ones, \vv = - \vz / \gamma - \ones$ and $w = -t / \gamma + 1$ for the dual (\ref{entropic_regularized}), leading to the equivalent dual form 
\begin{align*} 
        \min_{\vu, \vv, w} &\sum_{i, j=1}^{n} \exp \left(- C_{i, j} / \gamma + u_i + v_j + w \right) + \sum_{i=1}^{n} \exp(u_i) \\
        &+ \sum_{j=1}^{n} \exp(v_j) - \inner{\vu}{\vr} - \inner{\vv}{\vc} - ws.
\end{align*}
This transformation will facilitate the bounding of the dual variables in later steps. 

\textbf{Step 2}: we proceed to bound the $\ell_\infty$-norm of the transformed optimal dual variables $\|(\vu^\ast, \vv^\ast, w^\ast)\|_\infty$.  Conventional analyses for OT such as \citep{lin2019efficient} are inapplicable to the case of POT due to the addition of the third dual variable $w$ and more intricate dependencies of the dual variables $\vu, \vv$. To this end, our novel proof technique establishes the tight bound of $\|(\vu^\ast, \vv^\ast, w^\ast)\|_\infty = \mathcal{\widetilde{O}}\left( \norm{\vC}_{\max} \right)$, which consequently translates to the final bound for original dual variables $\norm{(\vy^\ast, \vz^\ast, t^\ast)}_2 = \mathcal{\widetilde{O}}\left(\sqrt{n} \norm{\vC}_{\max} \right)$ in Lemma \ref{Bounds_for_(u,v)}. Bounding the $\ell_2$-norm (i.e. bounding $\bar{R}$) is crucial because it contributes to the APDAGD guarantees (Theorem \ref{theorem:APDAGD_guarantees}) and the final complexity.

\textbf{Step 3}: 
Combining the $\bar{R}$ bound from \textbf{Step 2} in view of \citep[Proposition 4.10]{lin2019efficient} and the theoretical guarantees of \textsc{Round-POT} (Theorem \ref{prop:rounding}), we conclude the final computational complexity of APDAGD of $\widetilde{\mathcal{O}}(n^{2.5} / \varepsilon)$.
\end{proof}

%% file: Main_Paper/DE.tex
\section{Dual Extrapolation (DE)}
Our novel POT rounding algorithm permits the development of Dual Extrapolation (DE) for POT. From our analysis, DE is a first-order and parallelizable algorithm that can approximate POT distance up to $\varepsilon$ accuracy with $\mathcal{\widetilde O}(1/\varepsilon)$ parallel depth and $\mathcal{\widetilde O}(n^2/\varepsilon)$ total work. 
\subsection{Setup}
For each feasible $\vx$, we have $\norm{\vx}_1 = \norm{\vX}_1 + \norm{\vp}_1 + \norm{\vq}_1 = \norm{\vr}_1 + \norm{\vc}_1 - s$. We can normalize $\vx = \vx / (\norm{\vr}_1 + \norm{\vc}_1 -s), \: \vb = \vb / (\norm{\vr}_1 + \norm{\vc}_1 -s)$. These imply $\vx \in \Delta_{n^2+2n}$. The POT problem formulation \eqref{main_problem} is now updated as
\begin{equation} \label{POT_areaconvex}
    \min_{\vx \in \Delta_{n^2+2n}} \:  \inner{\vd}{\vx} ~~ \text{s.t.} ~~ \vA\vx = \vb,
\end{equation}
We then consider the $\ell_1$ penalization for the problem (\ref{POT_areaconvex}) and show that it has equal optimal value and $\varepsilon$-approximate minimizer to those of the POT formulation (\ref{POT_areaconvex}) (more details in $\ell_1$ Penalization subsection in Appendix). Through a primal-dual point of view, the $\ell_1$ penalized objective \eqref{l1_penalized} can be rewritten as
\begin{equation} \label{bilinear_objective}
    \min_{\vx \in \mathcal{X}} \max_{\vy \in \mathcal{Y}} \: F(\vx,\vy) \defeq \vd^\top \vx + 23 \norm{\vd}_\infty \left( \vy^\top \vA \vx - \vy^\top \vb \right), 
\end{equation}
with $\mathcal{X} = \Delta_{n^2+2n}, \mathcal{Y} = [-1,1]^{2n+1}$. Note that the term $23$ comes from the guarantees of  $\textsc{Round-POT}$ (Theorem \ref{prop:rounding}, Lemma \ref{lem:l1_regression}). Let $\mathcal{Z} = \mathcal{X} \times \mathcal{Y}$ such that $\vx \in \mathcal{X}$ and $\vy \in \mathcal{Y}$. For a bilinear objective $F(\vx,\vy)$ that is convex in $\vx$ and concave in $\vy$, it is natural to define the gradient operator $g(\vx,\vy) = (\nabla_{\vx} F(\vx,\vy),- \nabla_{\vy} F(\vx,\vy))$. Specifically for the objective \eqref{bilinear_objective}, we have $g(\vx,\vy) = (\vd + 23 \norm{\vd}_\infty \vA^\top \vy, -23 \norm{\vd}_\infty (\vA \vx -\vb)).$  This minimax objective can be solved with the dual extrapolation framework \citep{nesterov_2006}, which requires strongly convex regularizers. This setup can be relaxed with the notion of area-convexity \citep[Definition 1.2]{Sherman-2017-Area}, in the following definition.
\begin{definition}
    \label{defn:area_convex}
    (Area Convexity) A regularizer $r$ is $\kappa$-area-convex w.r.t an operator g if for any $x_1, x_2, x_3$ in its domain,
    \begin{align*}
        &\kappa \sum_{i}^3 r(x_i) - 3\kappa r \left(\frac{\sum_{i}^3 x_i}{3}\right) \geq \inner{g(x_2)- g(x_1)}{x_2-x_3}.
    \end{align*}
\end{definition}
The regularizer chosen for this framework is the Sherman regularizer, introduced in \citep{Sherman-2017-Area}
\begin{equation} \label{regularizer}
    r(\vx,\vy) = 2 \norm{\vd}_\infty \left( 10 \inner{\vx}{\log{\vx}} + \vx^\top \vA^\top (\vy^2) \right),
\end{equation}
in which $\vy^2$ is entry-wise. While this regularizer has a similar form to that in \citep{Jambulapati-2019-Direct}, the POT formulation leads to a different structure of $\vA$. For instance, $\norm{\vA}_1 = 3$ instead of $2$. This following lemma shows that chosen $r$ is 9-area-convex (its proof is in the Proof of Lemma \ref{9_area_convex} section in the Appendix).
\begin{lemma} \label{9_area_convex}
    $r$ is 9-area-convex with respect to the gradient operator $g$, i.e., $\kappa = 9$.
\end{lemma}

\subsection{Algorithmic Development}

\begin{algorithm}[t]
    \caption{Dual Extrapolation for POT}
    \label{alg:Dual_Extrapolation_POT}
    \begin{algorithmic}[1]
    \REQUIRE linearized cost $\vd$; linear operator $\vA$; constraints $\vb$; area-convexity coefficient $\kappa$; initial states \(\vs^0_{\vx} = \zeros_{n^2+2n}, \vs^0_{\vy} = \zeros_{2n+1}\); iterations M (Theorem \ref{Complexity_for_AM})
    \STATE \(\nabla_\vx r(\bar{\vz}) = 20 \norm{
    \vd}_\infty (1 - \log(n^2+2n)) \ones_{n^2 +2n}\)
    \STATE \(\nabla_\vy r(\bar\vz) = \zeros_{2n+1}\)
    \FOR{$t=0,1,2,\ldots,T-1$}
        \STATE \(\vv = \vs^t_{\vx} - \nabla_\vx r(\bar{\vz})\)
         \STATE \(\vu =  \vs^t_{\vy} - \nabla_\vy r(\bar{\vz})\)
        \STATE \( (\vz_\vx^t, \vz_\vy^t) = \mathtt{AM}(M, \vv, \vu)\) 
        \STATE \(\vv = \vv + (\vd + 23 \norm{\vd}_\infty \vA^\top \vz^t_{\vy}) / \kappa\) 
        \STATE \(\vu = \vu - 23 \norm{\vd}_\infty ( \boldsymbol A \vz_{\vx}^t - \vb) / \kappa\) 
        \STATE \( (\vw_{\vx}^t, \vw_{\vy}^t) = \mathtt{AM}(M, \vv, \vu)\)
        \STATE \(\vs^{t+1}_\vx = \vs^t_\vx + (\vd + 23 \norm{\vd}_\infty \vA^\top \vw_{\vy}^t) / (2 \kappa) \)
         \STATE \(\vs^{t+1}_{\vy} = \vs^{t}_{\vy} - 23 \norm{\vd}_\infty( \boldsymbol A \vw_{\vx}^t -\vb) / (2 \kappa) \)
    \ENDFOR
    \ENSURE 
        $ \bar{\vw}_{\vx} = \sum_{t=0}^{T-1} \vw_{\vx}^{t} / T$, $\bar{\vw}_{\vy} = \sum_{t=0}^{T-1} \vw_{\vy}^{t} / T$.
    \end{algorithmic}
\end{algorithm}

The main motivation is the DE Algorithm \ref{alg:General_Dual_Extrapolation} in Appendix, proposed by \citep{nesterov_2006}. This general DE framework essentially has two proximal steps per iteration, while maintaining a state $\vs$ in the dual space.
We follow \citep{Jambulapati-2019-Direct} and update $\vs$ with $1/2 \kappa$ rather than $1/ \kappa$ \citep{nesterov_2006}. The proximal steps (steps 2 and 3 in Algorithm \ref{alg:General_Dual_Extrapolation}) needs to minimize: 
\begin{equation} \label{General_AM}
    P(\vx, \vy) \defeq \inner{\vv}{\vx} + \inner{\vu}{\vy} + r(\vx,\vy).
\end{equation}
This can be solved efficiently with an Alternating Minimization ($\mathtt{AM}$) approach \citep{Jambulapati-2019-Direct}. Details for $\mathtt{AM}$ are included in the Appendix.  
Combining both algorithms, we have the DE for POT Algorithm \ref{alg:Dual_Extrapolation_POT}, where each of the two proximal steps is solved by the $\mathtt{AM}$ subroutine.

\subsection{Computational Complexities}
Firstly, we bound the regularizer $r$ to satisfy the convergence guarantees \citep{Jambulapati-2019-Direct} in Proof of Lemma \ref{T} subsection in Appendix. In that same subsection, we also derive the required number of iterations $T$ in DE with respect to $\Theta$, the range of the regularizer. Next, we have this essential lemma that bounds the number of iterations to evaluate a proximal step.
\begin{theorem} \label{Complexity_for_AM}
    (Complexity of $\mathtt{AM}$) For $T = \lceil 36 \Theta/\varepsilon \rceil$ iterations of DE, $\mathtt{AM}$ Algorithm \ref{alg:alternating_minimization} obtains additive error $\varepsilon / 2$ in 
    \begin{equation*}
        M = 24 \log \left( \left( 840 \norm{\vd}_\infty / \varepsilon^2 + 6 / \varepsilon \right) \Theta + 1336 \norm{\vd}_\infty / 9 \right)
    \end{equation*}
    iterations. This is done in wall-clock time $\mathcal{O}(n^2 \log \eta)$ with $\eta = \log n \norm{\vd}_\infty / \varepsilon$.
\end{theorem}
The proof of this theorem can be found in Proof of Theorem \ref{Complexity_for_AM} subsection in Appendix. The main proof idea is to  bound the number of iterations required to solve the proximal steps. This explicit bound for the number of $\mathtt{AM}$ iterations is novel as in DE for OT \cite{Jambulapati-2019-Direct}, the authors runs a while loop and do not analyze its final number of iterations. We can now calculate the computational complexity of the DE algorithm. The proof is in Proof of Theorem \ref{theorem:DE_complexity} subsection in Appendix.
\begin{theorem} (Complexity of DE)
    \label{theorem:DE_complexity}
    In $\mathcal{\widetilde O}(n^2 \norm{\vC}_{max} / \varepsilon)$ wall-clock time, the DE Algorithm \ref{alg:Dual_Extrapolation_POT} returns $(\bar{\vw}_{\vx}, \bar{\vw}_{\vy}) \in \mathcal{Z}$ so that the duality gap \eqref{duality_gap} is less than $\varepsilon$.
\end{theorem}


%% file: Main_Paper/experiments.tex
\section{Numerical Experiments}
In this section, we provide numerical results on approximating POT and its applications using the algorithms presented above.\footnote{Implementation and numerical experiments can be found at \url{https://github.com/joshnguyen99/partialot}.} In all settings, the optimal solution is found by solving the linear program \eqref{eq:pot_formulation} using the \texttt{cvxpy} package. Due to space constraint, we also include many extra experiments such as domain adaptation application, large-scale APDAGD, run time for varying $\epsilon$ comparison, and revised Sinkhorn performance in Appendix, further solidifying the efficiency and practicality of the proposed algorithms.
\begin{figure}
    \centering
    \includegraphics[width=0.45\textwidth]{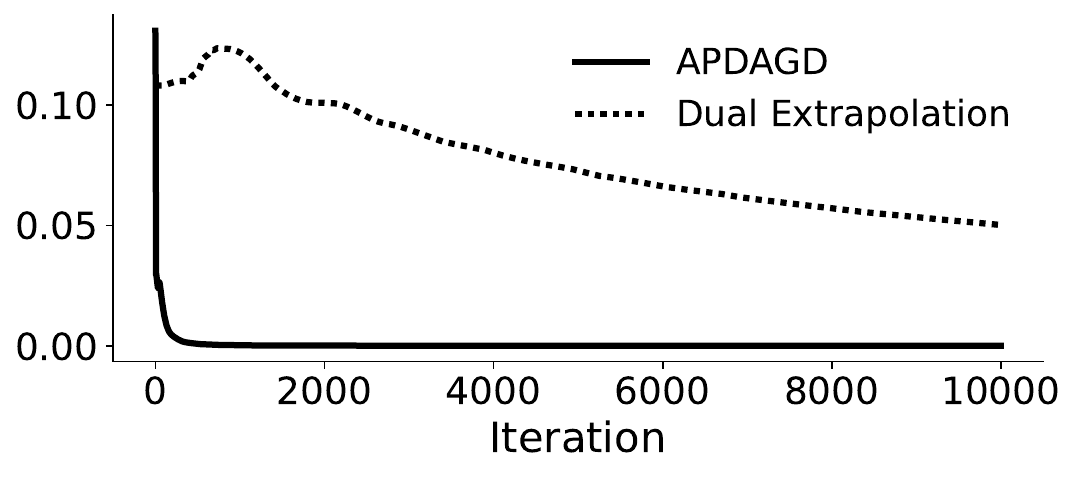}
    \caption{Primal optimality gap for solutions produced by APDAGD and Dual Extrapolation.}
    \label{fig:de_vs_apdagd}
\end{figure}
\begin{figure}
    \centering
    \includegraphics[width=0.9\linewidth]{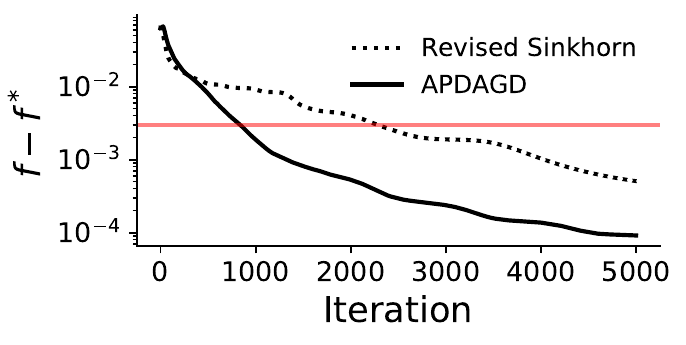}
    \caption{Primal optimality gap against optimization rounds achieved by our \textbf{revised} Sinkhorn and APDAGD for POT. The marginal distributions are taken from the later color transfer application in this section, and the red horizontal line depicts the pre-defined tolerance ($\varepsilon$) for both algorithms.}
    \label{fig:Feasible_Sinkhorn}
\end{figure}
\subsection{Run Time Comparison}
\textbf{APGAGD vs DE}: We are able to implement the DE algorithm for POT while the authors of DE for OT faced numerical overflow and had to use mirror prox instead “for numerical stability considerations” \cite{Jambulapati-2019-Direct}. For the setup of Figure \ref{fig:de_vs_apdagd}, we use images in the CIFAR-10 dataset \citep{krizhevsky2009learning} as the marginals. More details are included in the Further Experiment Setup Details section in Appendix. In Figure \ref{fig:de_vs_apdagd}, despite having a better theoretical complexity, DE has relatively poor practical performance compared to APDAGD. This is not surprising as previous works on this class of algorithms like \citep{pmlr-v130-dvinskikh21a} reach similar conclusions on DE's practical limitations. This can be partly explained by the large constants that are dismissed by the asymptotic computational complexity. Thus, in our applications in the following subsection, we will use APDAGD or revised Sinkhorn.

\textbf{APGAGD vs Sinkhorn}: For the same setting with CIFAR-10 dataset, we report that the ratio between the average per iteration cost of APDAGD and that of Sinkhorn is 0.68. Furthermore, in the Run Time for Varying $\varepsilon$ section in Appendix, we reproduce the same result in Figure 6 as \textbf{Figure 1 in (Dvurechensky et al., 2018)}, comparing the runtime of APDAGD and Sinkhorn for varying $\varepsilon$.

\subsection{Revised Sinkhorn}
Using the similar setting as the later example of color transfer, we can empirically verify that our revised Sinkhorn procedure can achieve the required tolerance $\varepsilon$ of the POT problem in Figure  (as opposed to pre-revised Sinkhorn in Figure \ref{fig:Infeasible_Sinkhorn}).

\subsection{Point Cloud Registration}
We now present an application of POT in point set registration, a common task in shape analysis. Start with two point clouds in three dimensions $R = \{ \vx_i \in \RR^3 \mid i = 1, \ldots, m \}$ and $Q = \{ \vy_j \in \RR^3 \mid j = 1, \ldots, n \}$. The objective is to find a transformation, consisting of a rotation and a translation, that best aligns the two point clouds. When the initial point clouds contain significant noise or missing data, the registration result is often badly aligned. Here, we consider a scenario where one set has missing values. In Figure \ref{fig:point_cloud} (a), the blue point cloud is set to contain the front half of the rabbit, retaining about 45\% of the original points. A desirable transformation must align the first halves of the two rabbits correctly.
Figure \ref{fig:point_cloud} compares the point clouds registration result using these methods. If $\vT$ is simply the OT matrix, not subject to a total transported mass constraint, the blue cloud is clearly not well-aligned with the red cloud. This is because the points for which the blue cloud is missing (i.e., the right half of the rabbit) are in the red cloud. If the total mass transported is set to $s = \frac{\min\{m, n\}}{\max\{m, n\}}$, then we end up with a POT matrix $\vT$ with all entries summing to $s$. In Figures \ref{fig:point_cloud} (c) and (d), the POT solution leads to a much better result than the OT solution: the left halves of the rabbit are closer together. Importantly, the feasible solution obtained by APDAGD leads to an even better alignment, compared to the pre-revised Sinkhorn, due to its infeasibility.

\begin{figure}
    \centering
    \includegraphics[width=1\linewidth]{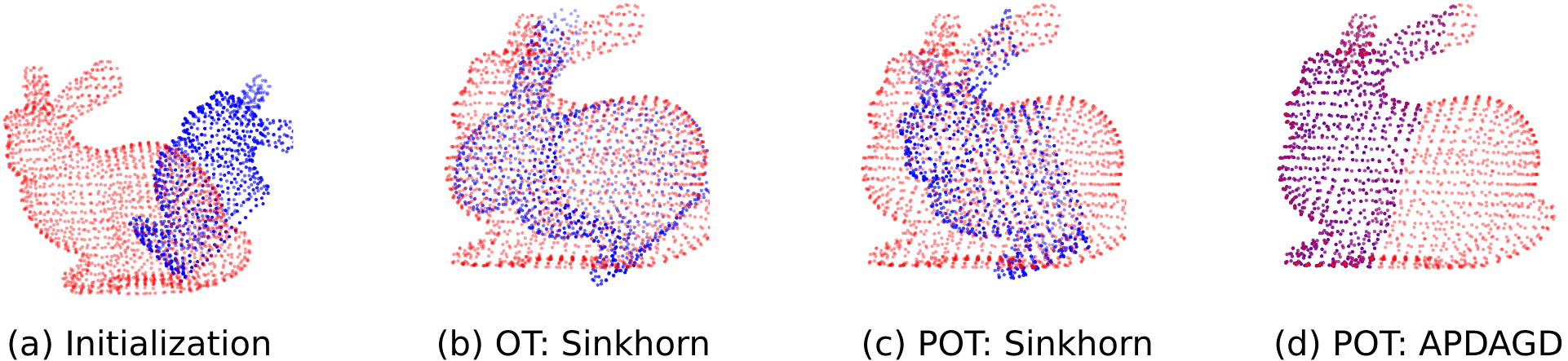}
    \caption{Point cloud registration using (partial) optimal transport. (a): Initial point sets with one set (in blue) missing 45\% of the points. (b): Registration result obtained after transforming the blue point cloud using the OT plan by Sinkhorn. (c): Registration result using the POT plan by \textbf{pre-revised} Sinkhorn \cite{nhatho-mmpot}. (d): Registration result using the POT plan by APDAGD.}
    \label{fig:point_cloud}
\end{figure}

\subsection{Color Transfer}
A frequent application of OT in computer vision is color transfer. POT offers flexibility in transferring colors between two possibly different-sized images, in contrast to OT which requires two color histograms to be normalized \citep{bonneel2015sliced}. We follow the setup by \citep{Blondel-2018-Smooth} and the detail of the implementation is described in the Further Experiment Setup Details section in Appendix. Our results are presented in Figure \ref{fig:color_transfer_comparison}. The third row of Figure \ref{fig:color_transfer_comparison} displays the resulting image produced by APDAGD with different levels of $\alpha$ (or transported mass $s$). Increasing $\alpha$ makes the wall color closer to the lighter part of the kiwi in the target image but comes at a cost of saturating the window frames' white. We emphasize that $\alpha$ is a tunable parameter, and the user can pick the most suitable level of transported mass. This can be more flexible than the vanilla OT which stringently requires all marginal masses to be normalized. We also emphasize that qualitatively, the solution produced by APDAGD closely matches the exact solution in (d), in contrast to \textbf{pre-revised} Sinkhorn's \cite{nhatho-mmpot}. This highlights the impact of having a feasible and optimal mapping on the quality of the application. 

\begin{figure}
    \centering
  \includegraphics[width=1\linewidth]{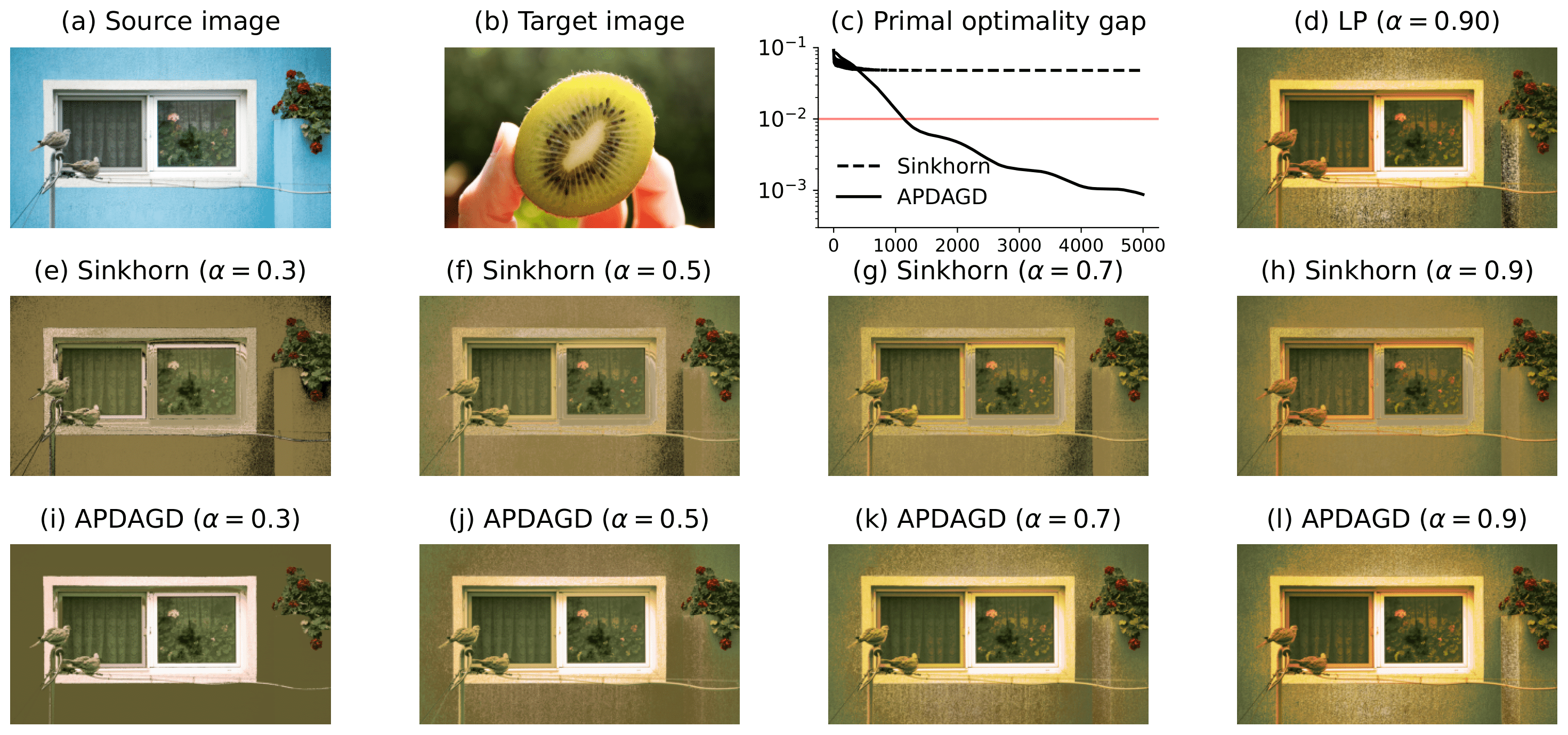}
  \caption{Color transfer using POT. The source and target images are of different sizes, giving rise to unbalanced color histograms. We apply Sinkhorn and APDAGD to solve the entropic regularized POT problem between these histograms. The optimality gap $\inner{\vC}{\vX} - f^*$ after each iteration is displayed in (c). The exact solution, retrieved by the Gurobi solver is presented in (d). On the second and third rows, the transformed image is displayed for different levels of $\alpha$, showing the flexibility of POT as opposed to OT.}  
  \label{fig:color_transfer_comparison}
\end{figure}

%% file: Main_Paper/conclusion.tex
\section{Conclusion}
In this paper, we first examine the infeasibility of Sinkhorn for POT. We then propose a a novel rounding algorithm which facilitates our development of feasible Sinkhorn procedure with guarantees, APDAGD and DE. DE achieves the best theoretical complexity while APDAGD and revised Sinkhorn are practically efficient algorithms, as demonstrated in our extensive experiments. We believe the rigor and applicability of our proposed methods will further facilitate the practical adoptions of POT in machine learning applications.

\section{Acknowledgements}
We would like to thank the reviewers of AAAI 2024 for their detailed and insightful comments and Dr. Darina Dvinskikh for sharing the numerical implementations of their work \citep{pmlr-v130-dvinskikh21a}. 


%% file: Appendix/appendix.tex
\section{Additional Related Literature}
Another alternative to OT for unbalanced measures, Unbalanced Optimal Transport (UOT) regularizes the objective by penalizing the marginal constraints via some given divergences \citep{UOT_complexity, Liero_Optimal_2018} including Kullback-Leiber \citep{chizat2017scaling} or squared $\ell_2$ norm \citep{Blondel-2018-Smooth}. Several works have studied UOT both in terms of computational complexity \citep{UOT_complexity, gem-uot} and applications \citep{Janati_Wasserstein_2019, balaji2020robust, Yang_Scalable_2019}. In practical applications, the structure of UOT and recent results in the approximation of OT with UOT \citep{gem-uot} suggest that UOT is best used in the time-varying regimes 
or strict control of mass is not required. For these reasons, UOT and POT cannot be directly compared to each other, and each serves separate purposes in different ML settings. The relaxation of the marginal constraints---which are strictly imposed by the standard OT and UOT---and the control over the total mass transported have granted POT immense flexibility compared to UOT. 

Similar to OT, POT or entropic regularized POT problems belong to the class of minimization problems with linear constraints. Given the structure of linear constraints and the property of strong duality, the prevalent method is to next formulate the Lagrange dual problem. It is then intuitive to solve the dual problem with first-order methods like Quasi-Newton \citep{dennis1977quasi} or Conjugate Gradient \citep{polyak1969conjugate}. Indeed, previous works \citep{cuturi2016smoothed, blondel2018smooth} utilized L-BFGS method to find the OT dual Lagrange problem. Based on the Accelerated Gradient Descent 
method with primal-dual updates, APDAGD  \citep{Dvurechensky-2018-Computational} gives better computational complexity and practicality compared to L-BFGS and other previous first-order methods applied to OT such as Quasi-Newton or Conjugate Gradient. With these APDAGD's theoretical and practical performances, we naturally want to develop a first-order framework based on AGDAGD to solve the POT problem. Nevertheless, there is ample room to further extend this first-order framework. For instance, one may look to develop a stochastic solver to handle stochastic gradient computations \cite{hoang_nonlinear_sa} or extend this method to decentralized settings such as \cite{lan2021graph}.


\section{Extra Algorithms}
\subsection{Sinkhorn}
\begin{algorithm}
    \caption{Sinkhorn for POT}
    \label{alg:Sinkhorn}
    \begin{algorithmic}[1]
        \REQUIRE{cost matrix $\vC$, marginals $\vr, \vc$, regularization term $\gamma$, transported mass $s$.}
        \STATE $\vC \gets \widetilde{\vC}, \vr \gets \widetilde{\vr}, \vq \gets \widetilde{\vq}$ (Updates follow the discussion in Revisiting Sinkhorn section) 
        \STATE $\widetilde{\vX} \gets \textsc{Sinkhorn}(\widetilde{\vC}, \widetilde{\vr}, \widetilde{\vc}, \gamma)$
        \STATE $\bar{\vX} \gets \textsc{Round}(\widetilde{\vX}, \mathcal{U}(\widetilde{\vr}, \widetilde{\vc}))$ \citep[Algorithm 2]{altschuler2017near}
        \ENSURE $\bar{\vX}[1:n,1:n]$
    \end{algorithmic}
\end{algorithm}
\subsection{APDAGD}
\begin{algorithm}
    \caption{Adaptive Primal-Dual Accelerated Gradient Descent (APDAGD) \citep[Algorithm 3]{Dvurechensky-2018-Computational}}
    \label{alg:APDAGD}
    \begin{algorithmic}[1]
        \REQUIRE accuracy $\varepsilon_f, \varepsilon_{eq} > 0$; initial estimate $L_0$ such that $0 < L_0 < 2L$; initialize $M_{-1} = L_0$, $i_0 = k = \alpha_0 = \beta_0 = 0$, $\pmb{\eta}_0 = \pmb{\zeta}_0 = \pmb{\lambda}_0 = \zeros$. 
        \WHILE{$f(\hat{\vx}_{k+1})+\varphi(\pmb{\eta}_{k+1}) \leq \varepsilon_f, \norm{\vA \hat{\vx}_{k+1} -\vb}_2 \leq \varepsilon_{eq}$}
            \WHILE{$\varphi(\pmb{\eta}_{k+1}) \leq \varphi(\pmb{\lambda}_{k+1}) + \inner{\nabla \varphi(\pmb{\lambda}_{k+1})}{\pmb{\eta}_{k+1} - \pmb{\lambda}_{k+1}} + \dfrac{M_k}{2}\norm{\pmb{\eta}_{k+1} - \pmb{\lambda}_{k+1}}_2^2$}
                \STATE \( M_k = 2^{i_k-1} M_k\) \\
                find $\alpha_{k+1}$ such that $\beta_{k+1} \defeq \beta_{k} + \alpha_{k+1} = M_k \alpha_{k+1}^2$
                \STATE \( \tau_k = \dfrac{\alpha_{k+1}}{\beta_{k+1}} \) 
                \STATE \( \pmb{\lambda}_{k+1} = \tau_k \pmb{\zeta}_k + (1-\tau_k)\pmb{\eta}_k  \) 
                \STATE \( \pmb{\zeta}_{k+1} = \pmb{\zeta}_{k} - \alpha _{k+1} \nabla \varphi(\pmb{\lambda}_{k+1})\)
                \STATE \( \pmb{\eta}_{k+1} = \tau_k \pmb{\zeta}_{k+1} + (1-\tau_k) \pmb{\eta}_k \)
            \ENDWHILE 
            \STATE \( \hat{\vx}_{k+1} = \tau_k \vx (\pmb{\lambda}_{k+1}) + (1-\tau_k)\hat{\vx}_k \)
            \STATE \( i_{k+1} = 0, k = k+1\)
        \ENDWHILE
        \ENSURE \(\hat{\vx}_{k+1}, \pmb{\eta}_{k+1}\).
    \end{algorithmic}
\end{algorithm}

\subsection{General Dual Extrapolation}
\begin{algorithm}
    \caption{General Dual Extrapolation}
    \label{alg:General_Dual_Extrapolation}
    \begin{algorithmic}[1]
    \REQUIRE initial state $\vs_0 = \zeros$; area-convexity coefficient $\kappa$; gradient operator $g$; number of iterations $T$; regularizer $r$ with minimizer $\Bar{\vz}$.
        \FOR{$t = 0,1,2,\ldots, T-1$} 
            \STATE $\vz^t = \mathtt{prox}^r_{\Bar{\vz}}(\vs^t)$\\
            \STATE $\vw^t = \mathtt{prox}^r_{\Bar{\vz}}(\vs^t + \dfrac{1}{\kappa} g(\vz^t) )$\\
            \STATE $\vs^{t+1} = \vs^t + \dfrac{1}{2 \kappa} g(\vw^t)$
            \STATE $t = t+1$\\
        \ENDFOR
    \ENSURE $\Bar{\vw} \defeq \dfrac{1}{T-1}
    \sum_{t=0}^T \vw^t$.
    \end{algorithmic}
\end{algorithm}
In this context of this work, we can define the Bregman divergence as follows
\begin{definition}[Bregman divergence]
    For differentiable convex function $h(\cdot)$ and $\vx, \vy$ in its domain, the Bregman divergence from $\vx$ to $\vy$ is $V^h_\vx(\vy) \defeq h(\vy) - h(\vx) - \inner{\nabla h(\vx)}{\vy - \vx}$.
\end{definition}
In our context, $h(\cdot)$ is chosen to be $r(\cdot)$, which is area-convex \citep[Definition 1.2]{Sherman-2017-Area}. Similar to \citep[Definition 2.5]{Jambulapati-2019-Direct}, we can define the proximal operator 
in Dual Extrapolation as follow 
\begin{definition}[Proximal operator]
    For differentiable (regularizer) $h(\cdot)$ with $\vx, \vy$ in its domain, and $g$ in the dual space, the proximal operator is $\mathtt{prox}^h_{\vx}(g) \defeq \argmin_\vy \{ \inner{g}{\vy} + V^h_\vx(\vy)\}$.
\end{definition}

\subsection{Alternating Minimization}
\begin{algorithm}
    \caption{Alternating minimization for \eqref{General_AM}  ($\mathtt{AM}$ ($M$, $\vv$, $\vu$))}
    \label{alg:alternating_minimization}
    \begin{algorithmic}[1]
    \REQUIRE number of iterations \(M\); vectors $\vv, \vu$; initialize $\vx = \dfrac{1}{n^2+2n} \ones_{n^2+2n}$; $\vy= \zeros_{2n+1}$.
    \FOR{$t=0,1,2,\ldots,M-1$}
        \STATE \(\pmb{\zeta} = \dfrac{1}{20 \norm{\vd}_\infty} \vv + \dfrac{1}{10} \vA^\top (\vy^{t})^2\)
        \STATE \(
        \vx^{t+1} = \dfrac{\exp(-\pmb{\zeta})}{\sum_{j=1}^{n^2+2n}[\exp(-\pmb{\zeta})]_j}\)
        \STATE \( \vy^{t+1} = \min \left( \ones, \max \left(-\ones, \dfrac{-\vu}{4 \norm{\vd}_\infty \vA \vx^{t+1}} \right) \right)\) \\
        (All operations are element-wise) 
    \ENDFOR
    \ENSURE \( \vx^{t+1}, \vy^{t+1} \).
    \end{algorithmic}
\end{algorithm}



\section{Revisiting Sinkhorn for POT}
\subsection{Proof of Theorem \ref{contraint_violation}}

\begin{proof}[\unskip \nopunct]
    From the augmentation of the marginals \citep{Chapel-nips2020} and the guarantees of the rounding algorithm \citep[Algorithm 2]{altschuler2017near}, we have 
    \begin{align*}
        &\widetilde{\vX} \ones = \norm{\vr}_1 + \norm{\vc}_1 - s,  \widetilde{\vX}^\top \ones = \norm{\vr}_1 + \norm{\vc}_1 - s\\
        &\ones^\top \Bar{\vX} \ones + \norm{\widetilde{\vp}}_1 = \norm{\vr}_1,  \norm{\widetilde{\vq}}_1 + \widetilde{X}_{n+1,n+1} = \norm{\vc}_1 - s\\
        &\ones^\top \Bar{\vX} \ones + \norm{\widetilde{\vq}}_1 = \norm{\vc}_1, \norm{\widetilde{\vp}}_1 + \widetilde{X}_{n+1,n+1} = \norm{\vr}_1 - s.
    \end{align*}
    Hence we can make a quick substitution 
    \begin{align*}
        \norm{\widetilde{\vp}}_1 + \widetilde{X}_{n+1,n+1} = \norm{\vr}_1 - s = \ones^\top \Bar{\vX} \ones + \norm{\widetilde{\vp}}_1 - s,
    \end{align*}
    which implies the constraint violation $V$
    \begin{align*}
        V = \ones^\top \Bar{\vX} \ones- s = \widetilde{X}_{n+1,n+1}.
    \end{align*}
    Hence, we now proceed to bound $\widetilde{X}_{n+1,n+1}$.
    we thus proceed to lower-bound $\tilde{X}_{n+1,n+1}$. Note that the reformulated OT problem that Sinkhorn solves 
    \begin{equation*}
        \min_{\mathbf{\tilde{X}} \geq 0} \langle \mathbf{\tilde{X}},\mathbf{\tilde{C}} \rangle \quad \text{s.t. } \quad \: \mathbf{X} \mathbf{1} = \mathbf{\tilde{r}}, \: \mathbf{X}^\top \mathbf{1} = \mathbf{\tilde{c}}.
    \end{equation*}
    With a well known dual formulation for entropic OT \citep{lin2019efficient}, we can minimize the Lagrangian w.r.t. $\mathbf{\tilde{X}}$ and obtain
    \begin{equation*}
        \mathbf{\tilde{X}} = \exp \left( - \dfrac{\mathbf{\tilde{C}}}{\gamma} + \mathbf{u} \mathbf{1}^\top + \mathbf{1} \mathbf{v}^\top\right), 
    \end{equation*}
where $\mathbf{u}, \mathbf{v}$ are the dual variable associated with each constraints. Now consider only the $\tilde{X}_{n+1,n+1}$ element, we can deduce
\begin{equation*}
   \tilde{X}_{n+1,n+1} = \exp \left(-\frac{\tilde{C}_{n+1,n+1}}{\gamma} - u_{n+1} -  v_{n+1}\right) \geq \exp \left( - \dfrac{A}{\gamma} - \|\mathbf{u}\|_\infty - \|\mathbf{v}\|_\infty\right).
\end{equation*}
Similar to our Lemma \ref{lem:inf_norm_dual} for POT, a similar bound for the infinity norm of OT dual variables \citep[Lemma 3.2]{nhatho-mmpot} shows that $$\| \mathbf{u} \|_\infty, \| \mathbf{v} \|_\infty \leq R,$$ where $$R =  \dfrac{A}{\gamma} + \log(n) - 2 \log\left({\min_{1\leq i, j\leq n}\{r_i, c_j\}}\right).$$ With $\gamma = \varepsilon / (4 \log{n})$ set by Sinkhorn, we obtain a lower bound of  the constraint violation $\tilde{X}_{n+1,n+1}$
\begin{equation*}
   \tilde{X}_{n+1,n+1} \geq \exp\left(-\dfrac{12A \log{n}}{\varepsilon} - \mathcal{O}(\log{n})\right ). 
\end{equation*}
Next, we establish an upper bound for the violation. We consider a feasible transportation matrix $\mathbf{\hat{X}}$ in the form $\hat{X}_{i,j} = s / n^2$ for $i,j\in[n]$ and $\hat{X}_{i,j} = 0$ otherwise. As $\mathbf{\tilde{X}}$ is a Sinkhorn $\varepsilon$-approximate solution of POT, we have 
\begin{align*}
    \langle \mathbf{\tilde{C}}, \mathbf{\tilde{X}} \rangle - \gamma H(\mathbf{\tilde{X}}) &\leq \langle \mathbf{\tilde{C}}, \mathbf{\hat{X}} \rangle - \gamma H(\mathbf{\hat{X}}) + \varepsilon\\
    &= \frac{s}{n^2} \sum^{n}_{i,j} C_{ij} - \gamma s \log\frac{s}{n^2} + \varepsilon \\
    &\leq s \|\mathbf{C}\|_\infty - \gamma s \log\frac{s}{n^2} + \varepsilon.
\end{align*}
On the other hand, we also have 
\begin{align*}
    \langle \mathbf{\tilde{C}}, \mathbf{\tilde{X}} \rangle - \gamma H(\mathbf{\tilde{X}}) &\geq A \tilde{X}_{n+1,n+1} + \gamma \sum^{n+1}_{i,j} \tilde{X}_{ij} \log(\tilde{X}_{ij}) \\
    &\geq A \tilde{X}_{n+1,n+1} + \gamma \left(\sum^{n+1}_{i,j} \tilde{X}_{ij}\right) \log \left( \frac{1}{n+1} \sum^{n+1}_{i,j} \tilde{X}_{ij} \right)\\
    &= A \tilde{X}_{n+1,n+1} + \gamma \left(\sum^{n+1}_{i,j} \tilde{X}_{ij}\right) \log \left(\sum^{n+1}_{i,j} \tilde{X}_{ij} \right) - \gamma \left(\sum^{n+1}_{i,j} \tilde{X}_{ij}\right) \log \left( n+1\right)\\
    &\geq (A + \gamma - \gamma \log(n+1))\tilde{X}_{n+1,n+1} + \gamma \left( \sum^{n}_{i,j} \tilde{X}_{ij} - 1 \right) \\
    &- \gamma \left(\sum^{n}_{i,j} \tilde{X}_{ij}\right) \log \left( n+1\right),
\end{align*}
where the second steps follows from Jensen's inequality, and the last step follows from the fact that $y\log{y} \leq y-1 \: \forall \: y > 0$. Combining the above two inequalities we deduce
\begin{align*}
    \begin{split}
        \tilde{X}_{n+1,n+1} &\leq \frac{1}{A - \gamma (\log(n+1)-1)} \left( s \|\mathbf{C}\|_\infty - \gamma s \log\frac{s}{n^2} + \varepsilon \right. \\
        &\left.+ \gamma \left( \sum^{n}_{i,j} \tilde{X}_{ij} - 1 \right) - \gamma \left(\sum^{n}_{i,j} \tilde{X}_{ij}\right) \log \left( n+1\right) \right) \\
        &=\mathcal{O}\left(\frac{\|\mathbf{C}\|_\infty}{A}\right).
    \end{split}
\end{align*}
\end{proof}

\subsection{Proof of Theorem \ref{them:revised_sinkhorn_complexity}}
\begin{proof}[\nopunct]
From Theorem \ref{contraint_violation}, we have the constraint violation $V$ or $\tilde{X}_{n+1,n+1}$ satifying
\begin{align*}
    \mathcal{O}\left(\frac{\|\mathbf{C}\|_\infty}{A}\right) \leq  \tilde{X}_{n+1,n+1}.
\end{align*}
Hence, in order to achieve the $\varepsilon$ tolerance we need
\begin{align*}
    \mathcal{O}\left(\frac{\|\mathbf{C}\|_\infty}{A}\right) \leq \varepsilon,
\end{align*}
yielding the sufficient size of A to be
\begin{align*}
   A = \mathcal{O}\left(\frac{\|\mathbf{C}\|_\infty}{\varepsilon}\right).
\end{align*}
The computational complexity of (Infeasible) Sikhorn for POT \citep{nhatho-mmpot} is 
\begin{align*}
    \tilde{\mathcal{O}}\left(\frac{n^2 \|\mathbf{\tilde{C}}\|_\infty^2}{\varepsilon^2}\right) = \tilde{\mathcal{O}}\left(\frac{n^2 A^2}{\varepsilon^2}\right).
\end{align*}
Plugging our derived sufficient value of A into the Sinkhorn complexity, we obtain the revised complexity for Feasible Sinkhorn to be
\begin{align*}
    \tilde{\mathcal{O}}\left(\frac{n^2 \|\mathbf{C}\|_\infty^2}{\varepsilon^4}\right).
\end{align*}
\end{proof}
\section{Rounding Algorithm} 
\subsection{Proof of Lemma \ref{Guarantees_for_EP} (Guarantees for the Enforcing Procedure)}
\begin{proof}[\unskip \nopunct]
    In the first case, we have $\alpha = \dfrac{\norm{\vr}_1-s}{\norm{\vp'}_1} (\geq 0)$, implying
    \begin{align*}
        \norm{\bar{\vp}}_1 = \norm{\vp''}_1 = \alpha \norm{\vp'}_1 = \frac{\norm{\vr}_1-s}{\norm{\vp'}_1} \norm{\vp'}_1 = \norm{\vr}_1-s.
    \end{align*}
    Also, consider 
    \begin{align*}
        &\bar{\vp} = \vp'' = \alpha \vp' \geq 0 \\
        &\bar{\vp} = \vp'' = \alpha \vp' \leq \vp' \leq \vr.
    \end{align*}
    So we have the guarantees as desired in the first scenario. For the second case, we obtain $\alpha = 1$, implying $\vp'' = \vp'$. Firstly we want to show the while loop will always stop. This is because the while condition $\norm{\vr}_1 - s \geq \norm{\vp''}_1 $ will be eventually violated:
    \begin{align*}
        \sum_{i=1}^n (r_i - p''_i) = \norm{\vr}_1 - \norm{\vp''}_1 \geq \norm{\vr}_1 - s - \norm{\vp''}_1. 
    \end{align*}
    During the updates, for some index $1 \leq j \leq i$, $p''_j = r_j$, while the rest remain the same. Therefore, we still have $\zeros \leq \vp'' \leq \vr$ after the running the while loop. For the last updated index $i$, we check that $0 \leq p''_i \leq r_i$, and the second inequality is obvious. The first inequality holds because we have $r_i \geq r_i - p'_i \geq \norm{\vp''}_1 - \norm{\vr}_1 +s$ as $p'_i = p''_i$ before the update inside the while loop. After the while loop, by updating $ p''_i = p''_i - (\norm{\vp''}_1 - \norm{\vr}_1 +s)$, we now have  $\norm{\vp''}_1 = \norm{\vr}_1 - s$ as desired.
\end{proof}

\subsection{Proof of Theorem \ref{prop:rounding} (Guarantees for the Rounding Algorithm)}
\begin{proof}[\unskip \nopunct]
    First, we show that $\vA \Bar{\vx} = \vb$. Note that $\norm{\ve_1}_1 = \norm{\ve_2}_1 = s - \ones^\top \vX' \ones$. From step 10, we have
    \begin{align*}
        &\Bar{\vX} \ones + \bar{\vp} = \left(\vX' + \dfrac{1}{\norm{\ve_1}_1} \ve_1 \ve_2^\top \right) \ones + \bar{\vp} = \vX' \ones + \ve_1 + \bar{\vp}= \vr, \\
        &\Bar{\vX}^\top \ones + \bar{\vq}= \left(\vX' + \dfrac{1}{\norm{\ve_1}_1} \ve_1 \ve_2^\top \right)^\top \ones + \bar{\vq}= \vX'^\top \ones + \ve_2 + \bar{\vq}= \vc,
    \end{align*}
    this implies
    \begin{equation*}
        \ones^\top \Bar{\vX} \ones = \ones^\top (\vr - \bar{\vp}) = \norm{\vr}_1 - \norm{\bar{\vp}}_1 = s.
    \end{equation*}

Next, we endeavor to bound $\norm{\vx - \bar{\vx}}_1$. Considering from the given information $\delta \geq \norm{\vA \vx - \vb}_1$, we obtain 
\begin{align*}
    \delta &\geq \norm{\vA \vx - \vb}_1 \\
    &= \norm{\vX \ones + \vp - \vr}_1 + \norm{\vX^\top \ones + \vq - \vc}_1 + |\ones^\top \vX \ones - s| \\
    &\geq \norm{\vX \ones + \vp - \vr}_1 + \norm{\vX^\top \ones + \vq - \vc}_1. 
\end{align*}
Since steps 7 through 10 have similar structure to \citep[Algorithm 2]{altschuler2017near} with changes in the marginals, by \citep[Lemma 7]{altschuler2017near}, we have 
\begin{align*}
    \norm{\vecflatten(\vX) - \vecflatten(\bar{\vX})}_1 &\leq 2 \left( \norm{\vX \ones + (\bar{\vp} - \vr)}_1 + \norm{\vX^\top \ones + (\bar{\vq} - \vc)}_1 \right) \\
    &\leq 2 \left( \norm{\vX \ones + \vp - \vr}_1 + \norm{\bar{\vp} - \vp}_1 + \norm{\vX^\top \ones + \vq - \vc}_1 + \norm{\bar{\vq} - \vq}_1\right) \\
    &\leq 2 (\delta + \norm{\bar{\vp} - \vp}_1 + \norm{\bar{\vq} - \vq}_1).
\end{align*}
Hence,
\begin{align}
    \norm{\vx - \bar{\vx}}_1 &= \norm{\vecflatten(\vX) - \vecflatten(\bar{\vX})}_1 + \norm{\bar{\vp} - \vp}_1 + \norm{\bar{\vq} - \vq}_1 \\
    &\leq 2 \delta + 3 (\norm{\bar{\vp} - \vp}_1 + \norm{\bar{\vq} - \vq}_1) \label{bound for x_bar}. 
\end{align}
Note that 
\begin{equation*}
    \norm{\bar{\vp} - \vp}_1 \leq \norm{\vp - \vp'}_1 + \norm{\vp' - \vp''}_1 + \norm{\vp'' - \bar{\vp}}_1. 
\end{equation*}
For bounding $\norm{\vp - \vp'}_1$, consider
\begin{equation*}
    \norm{\vp - \vp'}_1 = \sum_{i=1}^n \left(p_i - \min\{p_i, r_i\}\right) = \sum_{i=1}^n \nu_i,
\end{equation*}
where 
\begin{align*}
    \nu_i = \begin{cases}
        0 & p_i \leq r_i \\
        p_i - r_i & p_i > r_i
    \end{cases}.
\end{align*}
If $p_i > r_i$, then obviously $(X1)_i + p_i > r_i$ as $(X1)_i > 0$. Hence we obtain
\begin{align*}
    \norm{\vp - \vp'}_1 = \sum_{i=1}^n \nu_i = \sum_{i: p_i > r_i}^n \nu_i \leq \sum_{i: p_i > r_i}^n \left((X1)_i + p_i - r_i\right) \leq \norm{\vX \ones + \vp - \vr}_ 1.
\end{align*}
Next, consider 
\begin{align*}
    \norm{\vp' - \vp''}_1 &= \sum_{i=1}^n p'_i | 1- \alpha | \\
    &= \norm{\vp'}_1 - \min \{\norm{\vp'}_1,  \norm{\vr}_1 - s\} \\
    &= \norm{\vp'}_1 - \frac{1}{2}(\norm{\vp'}_1+ \norm{\vr}_1 - s - |\norm{\vp'}_1 - \norm{\vr}_1 + s|),
\end{align*}
Further simplifying the RHS, we have
\begin{align*}
    \norm{\vp' - \vp''}_1 &= \frac{1}{2} \left( \norm{\vp'}_1 - \norm{\vr}_1 + s + |\norm{\vp'}_1 - \norm{\vr}_1 + s| \right) \\
    &\leq |\norm{\vp'}_1 - \norm{\vr}_1 + s| \\
    &\leq |\norm{\vp}_1 - \norm{\vr}_1 + s| + \norm{\vp}_1 - \norm{\vp'}_1 \\
    &= |\norm{\vp}_1 - \norm{\vr}_1 + s| + \norm{\vp - \vp'}_1 \\
    &\leq |\norm{\vp}_1 - \norm{\vr}_1 + s| + \norm{\vX \ones + \vp - \vr}_1.
\end{align*}
Finally, we bound for  $\norm{\vp'' - \bar{\vp}}_1$. Note that 
\begin{align*}
    \norm{\vp'' - \bar{\vp}}_1 &= \norm{\vr}_1 - s - \norm{\vp''}_1  \\
    &= \norm{\vr}_1 - s - \norm{\vp}_1 + (\norm{\vp}_1 - \norm{\vp''}_1) \\
    &= \norm{\vr}_1 - s - \norm{\vp}_1 + \norm{\vp - \vp''}_1,
\end{align*}
yielding
\begin{align*}
    \norm{\vp'' - \bar{\vp}}_1 &\leq |\norm{\vp}_1 - \norm{\vr}_1 + s| + \norm{\vp - \vp'}_1 + \norm{\vp' - \vp''}_1 \\
    &\leq 2 |\norm{\vp}_1 - \norm{\vr}_1 + s| + 2 \norm{\vX \ones + \vp - \vr}_1.
\end{align*}
Combining the bounds for $\norm{\vp - \vp'}_1, \norm{\vp' - \vp''}_1, \norm{\vp'' - \bar{\vp}}_1$, we obtain
\begin{align*}
    \norm{\bar{\vp} - \vp}_1 \leq 3 |\norm{\vp}_1 - \norm{\vr}_1 + s| + 4 \norm{\vX \ones + \vp - \vr}_1.
\end{align*}
Similarly, we can obtain the bound for $\norm{\bar{\vq} - \vp}_1$
\begin{align*}
    \norm{\bar{\vq} - \vq}_1 \leq 3 |\norm{\vq}_1 - \norm{\vc}_1 + s| + 4 \norm{\vX^\top \ones + \vq - \vc}_1.
\end{align*}
Next, note that
\begin{align*}
    &\delta + |\ones^\top \vX \ones - s| \\
    &\geq \norm{\vX \ones + \vp - \vr}_1 + \norm{\vX^\top \ones + \vq - \vc}_1 + 2 |\ones^\top \vX \ones - s| \\
    &= |\ones^\top \vX \ones + \|\vp\|_1 - \|\vr\|_1| + |\ones^\top \vX \ones - s| + |\ones^\top \vX \ones + \|\vq\|_1 - \|\vc\|_1| + |\ones^\top \vX \ones - s| \\ 
    &\geq |\|\vp\|_1 - \|\vr\|_1 + s| + |\|\vq\|_1 - \|\vc\|_1 + s|.
\end{align*}
Hence we obtain
\begin{align*}
    &\norm{\bar{\vp} - \vp}_1 + \norm{\bar{\vq} - \vq}_1\\
    &\leq 3 (|\|\vp\|_1 - \|\vr\|_1 + s| + |\|\vq\|_1 - \|\vc\|_1 + s|) + 4 (\norm{\vX \ones + \vp - \vr}_1 +\norm{\vX^\top \ones + \vq - \vc}_1) \\
    &\leq 3 \delta + 3|\ones^\top \vX \ones - s| +  4 (\norm{\vX \ones + \vp - \vr}_1 +\norm{\vX^\top \ones + \vq - \vc}_1) \\
    &= 3 \delta + 3 \norm{\vA \vx -\vb}_1 + \norm{\vX \ones + \vp - \vr}_1 +\norm{\vX^\top \ones + \vq - \vc}_1 \\
    &\leq 7 \delta
\end{align*}
Therefore, from \eqref{bound for x_bar}, we conclude 
\begin{align*}
    \norm{\vx - \bar{\vx}}_1 \leq 23 \delta.
\end{align*}
\end{proof}
\section{Properties of Entropic POT}
The term $\inner{\vx}{\log \vx}$ is  the negative entropic regularizer for every $\vx \geq \zeros$, with the convention that $\zeros \log \zeros = \zeros$, and the parameter $\gamma > 0$ controls the strength of regularization. It is well-known that the objective $\varphi(\pmb{\lambda})$ is $L$-smooth with $L \leq \norm{\vA}^2_{1 \rightarrow 2} / \mu_f$ \citep{nesterov2005smooth}, since $f$ is $\mu_f$-strongly convex w.r.t. $\norm{\cdot}_1$, where $\mu_f = \gamma / (\| \vr \|_1 + \| \vc \|_1 - s)$ (Lemma \ref{lemma:strongly_convex}). Instead of this loose estimate, APDAGD uses line search to find and adapt the local Lipschitz value \citep{Dvurechensky-2018-Computational}. Further, the gradient of $\varphi$ is $\nabla\varphi(\pmb{\lambda}) = \vb - \vA \vx(\pmb{\lambda})$ where $\vx(\pmb{\lambda}) \defeq \argmax_{\vx \in Q} \left\{ - f(\vx) - \inner{\vx}{\vA^\top \pmb{\lambda}} \right\}$. 
\subsection{Strong Convexity}
\begin{lemma}
    \label{lemma:strongly_convex}
    For convenience, denote $\| \vr \|_1 + \| \vc \|_1 - s$ as $D$. Let $Q = \{\vx \in \RR^{n^2+2n} | \vx > \zeros,  \|\vx\|_{1} \leq D \}$. The entropic regularized objective function $f(\vx)= \sum_{i=1}^{n^2+2n} \gamma x_i log(x_i) + d_i x_i$ is $\dfrac{\gamma}{D}$-strongly-convex with respect to $\ell_1$ norm over $Q$. 
\end{lemma}
\begin{proof}
    To prove the Lemma, we first consider the following
    Observation \ref{strongconvex2_obs1}:
    \begin{observation}
    \label{strongconvex2_obs1}
        $f(\vx)$ is $\dfrac{\gamma}{D}$-strongly-convex with respect to $\ell_1$ norm over $Q$ if for all $\vx,\vy \in Q$,
        \begin{align}
        \label{strongconvex_alternative1}
            \langle \nabla f(\vx) - \nabla f(\vy), \vx-\vy \rangle \geq \frac{\gamma}{D} \|\vx-\vy \|_1^2.
        \end{align}
    \end{observation}
    \begin{proof}[Proof of Observation \ref{strongconvex2_obs1}]
            Consider $\vx \neq \vy$. For $\lambda \in [0, 1]$, we define  $h_1(\lambda)= f(\vy+\lambda(\vx-\vy)) - \lambda \langle \nabla f(\vy) , \vx-\vy \rangle$ and $h_2(\lambda) = \dfrac{\lambda^2 \gamma}{2D} \|\vx-\vy \|_1^2$, and further obtain:
    \begin{align*}
        h_1'(\lambda) &=  \langle \nabla f(\vy+\lambda(\vx-\vy)) , \vx-\vy \rangle-  \langle \nabla f(\vy), \vx-\vy \rangle \\
        &= \langle \nabla f(\vy+\lambda(\vx-\vy)) - \nabla f(\vy)  , \vx-\vy \rangle \\
        h_2'(\lambda) &= \frac{\lambda \gamma}{D} \|\vx-\vy\|_1^2. 
    \end{align*} 
    By given information from \eqref{strongconvex_alternative1}, we have:
    \begin{align*}
        \lambda h_1'(\lambda)&= \langle \nabla f(\vy+\lambda(\vx-\vy)) - \nabla f(\vy)  , [\vy+\lambda (\vx-\vy)] - \vy \rangle \\
        &\geq \frac{\gamma}{D} \| [\vy+\lambda (\vx-\vy)] - \vy \|_1^2 = \lambda h_2'(\lambda)  \\
        \therefore h_1'(\lambda) &\geq h_2'(\lambda).
    \end{align*}
By Cauchy's mean value theorem, there exists $c\in (0,1)$ such that: 
\begin{align}
    &h_1(1) - h_1(0) = \frac{h_1'(c)}{h_2'(c)}
    (h_2(1) - h_2(0)) \geq h_2(1) - h_2(0)\\
    \therefore \: &f(\vx) \geq f(\vy) + \langle \nabla f(\vy), \vx-\vy \rangle + \frac{\gamma}{2D} \| \vx-\vy\|_1^2. \label{strongconv1}
\end{align}
Hence, we have the desired conclusion for any $\vx \neq \vy$. We also have \eqref{strongconv1} to hold trivially for $\vx=\vy$.
    \end{proof}
Back to our main proof of Lemma \ref{lemma:strongly_convex}, we proceed to prove that for all $\vx,\vy \in Q$, the inequality \eqref{strongconvex_alternative1} holds. We consider $\vx \neq \vy$, while for $\vx = \vy$ \eqref{strongconvex_alternative1} trivially holds.
Since $\dfrac{\partial f(\vx)}{\partial x_i} = \gamma log(x_i) +\gamma + d_i$, we rewrite
\eqref{strongconvex_alternative1} as: 
\begin{align*}
    \sum_{i=1}^{k} \gamma (x_i-y_i) (log(x_i) - log(y_i)) \geq \frac{\gamma}{D} \|\vx-\vy \|_1^2.
\end{align*}
For $\lambda \in [0, 1]$, we define $s_1(\lambda) = \sum_{i=1}^{k} (x_i-y_i) log(y_i + \lambda (x_i - y_i)) $ and $s_2(\lambda) = \lambda \|\vx-\vy\|_1^2$, and further obtain:
\begin{align*}
    s_1'(\lambda) &= \sum_{i=1}^{k} \frac{(x_i-y_i)^2}{y_i + \lambda (x_i-y_i)} \\
    s_2'(\lambda) &= \|\vx-\vy\|_1^2.
\end{align*}

Using $\sum_{i=1}^{k} [y_i + \lambda (x_i-y_i)]= \sum_{i=1}^{k} [ \lambda x_i +(1-\lambda)y_i ] \leq \lambda D +(1-\lambda) D =D$, we have: 
\begin{align*}
    D \cdot s_1'(\lambda) &\geq \sum_{i=1}^{k} [ y_i + \lambda (x_i-y_i)] \cdot   \sum_{i=1}^{k} \frac{(x_i-y_i)^2}{y_i + \lambda (x_i-y_i)} \\
    &\geq (\sum_{i=1}^{k} |x_i-y_i|)^2 = \|\vx-\vy\|_1^2 \text{  (by Cauchy-Schwarz inequality)} \\
    &\geq s_2'(\lambda).
\end{align*}
By Cauchy's mean value theorem, there exists $c\in (0,1)$ such that: 
 \begin{align*}
 D(s_1(1) - s_1(0)) &= \frac{D\cdot s_1'(c)}{s_2'(c)} (s_2(1) - s_2(0)) \geq s_2(1) - s_2(0).
 \end{align*}
Since $s_1(1) - s_1(0)=\sum_{i=1}^{k} (x_i-y_i) (log(x_i) - log(y_i))$ and $s_2(1) - s_2(0)= \|\vx-\vy\|_1^2 $, we get:  
\begin{align*}
     \sum_{i=1}^{k} (x_i-y_i) (log(x_i) - log(y_i)) \geq \frac{1}{D} \|\vx-\vy\|_1^2
 \end{align*}
The condition \eqref{strongconvex_alternative1} thus follows for all $\vx,\vy \in Q$. By Observation \ref{strongconvex2_obs1}, $f(\vx)$ is $\dfrac{\gamma}{\| \vr \|_1 + \| \vc \|_1 - s}$-strongly-convex with respect to $\ell_1$ norm over $Q$.
\end{proof}

\subsection{Dual Formulation for Entropic POT}
First we note the general form of Equation (\ref{prob:entropic}). Given a simple closed convex set $Q$ in a finite-dimensional real vector space $E$; $H$ is another finite-dimensional real vector space with an element $\vb$; $H^\ast$ is the dual space of $H$. $\vA$ is a given linear operator from $E$ to $H$. Consider
\begin{equation} \label{general_problem}
    \min_{\vx \in Q \subseteq E} f(\vx) \quad \text{s.t. }  \vA \vx = \vb,
\end{equation}
where $f(\vx)$ is a $\mu_f$-strongly convex function on $Q$ with respect to some chosen norm $\norm{\cdot}_E$ on $E$. If $\gamma = 0$, we recover Problem \eqref{main_problem}. 

Let $\vy \in \RR^n,\vz \in \RR^n$ and $t \in \RR$ be the dual variables corresponding to the three equality constraints. The Lagrangian is
\begin{align*}
    \mathcal{L} & = \inner{\vC}{\vX} + \gamma (\inner{\vX}{\log \vX} + \inner{\vp}{\log \vp} + \inner{\vq}{\log \vq}) + \\
    &\inner{\vy}{\vX \ones +\vp - \vr} + \inner{\vz}{\vX^\top \ones +\vq - \vc} + t (\ones^\top \vX \ones - s),
\end{align*}
Hence, we have the dual problem:
\begin{align} \label{original_dual}
    \begin{split}
        \max_{\vy, \vz, t} & \left\{ - \inner{\vy}{\vr} - \inner{\vz}{\vc} - ts + \min_{X_{i, j} \geq 0} \sum_{i, j=1}^{n} X_{i, j}(C_{i, j} + \gamma \log X_{i, j} + y_i + z_j + t)\right. \\ 
        & \left. \quad + \min_{p_i \geq 0} \sum_{i=1}^{n} p_i (\gamma \log{p_i} + y_i) + \min_{q_j \geq 0}\sum_{j=1}^{n} q_j (\gamma \log{q_j} + z_j)\right\}.
    \end{split}
\end{align} 
Next, we minimize the Lagrangian w.r.t. $\vX$, $\vp$ and $\vq$ according to the first order condition, which gives
\begin{equation*}
    X_{i, j} = \exp \left( - \frac{1}{\gamma} \left( C_{i, j} + y_i + z_j + t \right) - 1 \right), p_i = \exp \left( - \frac{y_i}{\gamma} - 1 \right), q_j = \exp \left( - \frac{z_j}{\gamma} - 1 \right).
\end{equation*}
Plugging this back to the dual problem \eqref{original_dual}, we have
\begin{align*}
    \begin{split}
        \max_{\vy, \vz, t} & \left\{ - \inner{\vy}{\vr} - \inner{\vz}{\vc} - ts  - \gamma \sum_{i, j=1}^{n} \exp \left( - \frac{1}{\gamma} \left( C_{i, j} + y_i + z_j + t \right) - 1 \right) \right. \\ 
        & \left. \quad - \gamma \sum_{i=1}^{n} \exp \left( - \frac{y_i}{\gamma} - 1 \right) - \gamma \sum_{j=1}^{n} \exp \left( - \frac{z_j}{\gamma} - 1 \right) \right\},
    \end{split}
\end{align*}
as desired.
Plugging this back into the dual problem, suppressing all additive constants and diving the objective by $- \gamma$, we obtain the following new dual problem:
\begin{equation} \label{entropic_regularized_problem}
    \min_{\vu, \vv, w} \; h(\vu,\vv,w) \coloneqq \sum_{i, j=1}^{n} e^{w} e^{u_i} e^{- C_{i, j} / \gamma} e^{v_j} + \sum_{i=1}^{n} e^{u_i} + \sum_{j=1}^{n} e^{v_j} -  \inner{\vu}{\vr} - \inner{\vv}{\vc} - ws.
\end{equation}

\section{APDAGD}
\subsection{APDAGD Theoretical Guarantees}
We have the following guarantees for APDAGD , similar to \citep[Theorem 3]{Dvurechensky-2018-Computational}.
\begin{theorem}
    \label{theorem:APDAGD_guarantees}
    Let $\vx^\ast$ and $f^\ast$ be the optimal solution and optimal value of \eqref{prob:entropic}, respectively. Suppose that $\norm{\pmb{\lambda}^\ast}_2 \leq \bar{R}$, the outputs $\hat{\vx}_k$ and $\pmb{\eta}_k$ from Algorithm $\ref{alg:APDAGD}$ satisfy
    \begin{align*}
        f(\hat{\vx}_k) - f^\ast &\leq f(\hat{\vx}_k) + \varphi(\pmb{\eta}_k) \leq  \frac{16 \norm{\vA}^2_{1 \rightarrow 2} \bar{R}^2}{\mu_f k^2}, \\
        \norm{\vA \hat{\vx}_k - \vb}_2 &\leq \frac{16 \bar{R} \norm{\vA}^2_{1 \rightarrow 2}}{\mu_f k^2}, \\
        \norm{\hat{\vx}_k - \vx^\ast}_1 &   \leq \frac{8 \bar{R}\norm{\vA}_{1 \rightarrow 2}}{\mu_f k}. 
    \end{align*}
\end{theorem}

\subsection{Complexity of APDAGD for POT (Theorem \ref{APDAGD_complexity}) Detailed Proof}

\subsubsection{Step 1: Change of Dual Variables}
\label{step 1}
For the sake of convenience, we introduce the change of variables $\vu, \vv \in \RR^n$ and $\vw \in \RR$ such that $u_i = - y_i / \gamma - 1, v_j = - z_j / \gamma - 1$ and $w = -t / \gamma + 1$. After suppressing all additive constants and diving the objective by $- \gamma$, we obtain the equivalent dual problem:
\begin{align} \label{entropic_regularized_problem_transformed}
        \min_{\vu, \vv, w} \; h(\vu,\vv,w) &\defeq \sum_{i, j=1}^{n} \exp \left(- \frac{C_{i, j}}{\gamma} + u_i + v_j + w \right) + \sum_{i=1}^{n} \exp(u_i) + \\
        &\sum_{j=1}^{n} \exp(v_j) -  \inner{\vu}{\vr} - \inner{\vv}{\vc} - ws.
\end{align}
Let $\vu^\ast, \vv^\ast, w^\ast$ be the optimal solutions of the entropic regularized problem \eqref{entropic_regularized_problem_transformed}. Denote $u^\ast_{max} = \max_{1 \leq i \leq n} u^\ast_i, u^\ast_{min} = \min_{1 \leq i \leq n} u^\ast_i$ and similarly for $v^\ast_{max}$ and $v^\ast_{min}.$ Notice that $\forall \: 1 \leq i \leq n, \: \exp (u^\ast_i) = p_i \leq r_i.$ This implies $\forall \: 1 \leq i \leq n, \: u^\ast_i \leq \log{r_i} \leq \log{1} = 0.$ So, $u^\ast_{max} \leq 0$. Similarly, $ v^\ast_{max} \leq 0$.

\subsubsection{Step 2: Bound for Infinity Norm of Transformed Dual Variables}
\label{step 2}
\begin{lemma}
    \label{lem:inf_norm_dual}
    The upper bound for $\norm{(\vu^\ast,\vv^\ast,w^\ast)}_\infty$ is
    \begin{align*}
        \norm{(\vu^\ast,\vv^\ast, w^\ast)}_\infty \leq \frac{\norm{\vC}_{\max} \max(\norm{\vr}_1, \norm{\vc}_1)}{\gamma (\max(\norm{\vr}_1, \norm{\vc}_1)-s)} - \log{ \left( \min_{1 \leq i,j \leq n} (r_i, c_j) \right)}.
    \end{align*}
\end{lemma}
\begin{proof} 
    First we lower bound $w^\ast$, note that 
    \begin{equation*}
        s = \sum_{i,j = 1}^{n} \exp \left(\frac{- C_{i, j}}{\gamma} + u^\ast_i + v^\ast_j + w^\ast \right) \leq n^2 \exp \left( u^\ast_{max} + v^\ast_{max} + w^\ast \right).
    \end{equation*}
    Taking logs on both sides, we have
    \begin{equation*} 
        \log{s} \leq 2\log{n} + u^\ast_{max} + v^\ast_{max} + w^\ast.
    \end{equation*}
    Since $u^\ast_{max}, v^\ast_{max} \leq 0$, we obtain the desired lower bound
    \begin{equation*}
        w^\ast \geq \log{s} - 2\log{n} - u^\ast_{max} - v^\ast_{max} \geq \log{s} - 2\log{n}.
    \end{equation*}
    For the upper bound, note that if we let $h^\ast$ be the optimal objective value of $h(\vu,\vv,w)$ in the transformed dual problem \eqref{entropic_regularized_problem_transformed}, then by accounting for the additive constants the optimal value of the original dual problem \eqref{entropic_regularized_problem} will be 
    \begin{equation} \label{h*}
        - \gamma h^\ast + \gamma (\norm{\vr}_1 + \norm{\vc}_1 - s).
    \end{equation}
    Accompanying $\eqref{h*}$ with the fact that the primal objective $\eqref{prob:entropic}$ optimal value is equal to the dual objective optimal $\eqref{entropic_regularized_problem}$ value by strong duality, we have that
    \begin{equation*} 
         - \gamma h^\ast + \gamma (\norm{\vr}_1 + \norm{\vc}_1 - s) \geq 0,
    \end{equation*}
    which yields
    \begin{equation} \label{h*_bound}
        h^\ast \leq \norm{\vr}_1 + \norm{\vc}_1 - s.
    \end{equation}
    On the other hand, consider 
    \begin{align*}
        s + \sum_{i=1}^{n} e^{u^\ast_i} = \sum_{i,j = 1}^{n} \exp \left(\frac{- C_{i, j}}{\gamma} + u^\ast_i + v^\ast_j + w^\ast \right) + \sum_{i=1}^{n} e^{u^\ast_i} &= \sum_{i,j = 1}^{n} X_{i,j} + \sum_{i=1}^{n} p_i = \norm{\vr}_1, \\
        s + \sum_{j=1}^{n} e^{v^\ast_j} = \sum_{i,j = 1}^{n} \exp \left(\frac{- C_{i, j}}{\gamma} + u^\ast_i + v^\ast_j + w^\ast \right) + \sum_{j=1}^{n} e^{v^\ast_j} &= \sum_{i,j = 1}^{n} X_{i,j} + \sum_{j=1}^{n} p_j = \norm{\vc}_1.
    \end{align*}
    This leads to
    \begin{equation*}
        h^\ast = h(\vu^\ast, \vv^\ast, w^\ast) = -s + \norm{\vr}_1 + \norm{\vc}_1 - \inner{\vu^\ast}{r} - \inner{\vv^\ast}{c} - w^\ast s.
    \end{equation*}
    Together with the inequality \eqref{h*_bound}, 
    \begin{equation*}
        - \inner{\vu^\ast}{r} - \inner{\vv^\ast}{c} - w^\ast s \leq 0.
    \end{equation*}
    This implies
    \begin{equation*}
        - w^\ast s \leq \inner{\vu^\ast}{\vr} + \inner{\vv^\ast}{\vc} \leq u^\ast_{\max} \norm{\vr}_1 + v^\ast_{\max} \norm{\vc}_1 \leq (u^\ast_{\max} + v^\ast_{\max}) \max(\norm{\vr}_1, \norm{\vc}_1),
    \end{equation*}
    which implies 
    \begin{equation*}
        u^\ast_{\max} + v^\ast_{\max} \geq \frac{- w^\ast s}{\max(\norm{\vr}_1, \norm{\vc}_1)}.
    \end{equation*}
    We also know that 
    \begin{equation*}
        u^\ast_i + v^\ast_j + w^\ast - \frac{C_{ij}}{\gamma} \leq 0 \: \: \forall i,j.
    \end{equation*}
    Hence, combining the two results above, we obtain
    \begin{equation*}
        \frac{\norm{\vC}_{\max}}{\gamma} \geq  u^\ast_{\max} + v^\ast_{\max} + w^\ast \geq \frac{- w^\ast s}{\max(\norm{\vr}_1, \norm{\vc}_1)} + w^\ast,
    \end{equation*}
    yielding the desire upper bound 
    \begin{equation*}
        w^\ast \leq \frac{\norm{\vC}_{\max} \max(\norm{\vr}_1, \norm{\vc}_1)}{\gamma (\max(\norm{\vr}_1, \norm{\vc}_1)-s)}.
    \end{equation*}
    Note that $\forall \: 1 \leq j \leq n,$ we have
    \begin{align*}
        \left( \sum^n_{j=1} \exp \left( v^\ast_j \right) \right) \exp \left( u^\ast_i + w^\ast \right) &\geq \sum^n_{j=1} \exp \left( - \frac{C_{i,j}}{\gamma} + u^\ast_i + v^\ast_j + w^\ast \right) \\
        &= \sum^n_{j=1} X_{i,j} \\
        &= r_i \\
        &\geq \min_{1 \leq i,j \leq n} (r_i, c_j).
    \end{align*}
    Taking log on both sides, we get
    \begin{equation*}
        \log \left( \sum^n_{j=1} \exp \left( v^\ast_j \right) \right) + u^\ast_i + w^\ast \geq \log{ \left( \min_{1 \leq i,j \leq n} (r_i, c_j) \right)},
    \end{equation*}
    which implies
    \begin{align*}
        u^\ast_i &\geq \log{ \left( \min_{1 \leq i,j \leq n} (r_i, c_j) \right)} - \log \left( \sum^n_{j=1} \exp \left( v^\ast_j \right) \right) - w^\ast \\
        &\geq \log{ \left( \min_{1 \leq i,j \leq n} (r_i, c_j) \right)} - \frac{\norm{\vC}_{\max} \max(\norm{\vr}_1, \norm{\vc}_1)}{\gamma (\max(\norm{\vr}_1, \norm{\vc}_1)-s)}.
    \end{align*}
    Similarly, we also obtain the same lower bound for $v^\ast_j, 
    \: \forall 1 \leq i \leq n$. And as $\forall 1 \leq i,j \leq n, \: u^\ast_i, v^\ast_j \leq 0 $, we obtain 
    \begin{align*}
        \norm{(\vu^\ast,\vv^\ast)}_\infty &\leq \left| \log{ \left( \min_{1 \leq i,j \leq n} (r_i, c_j) \right)} - \frac{\norm{\vC}_{\max} \max(\norm{\vr}_1, \norm{\vc}_1)}{\gamma (\max(\norm{\vr}_1, \norm{\vc}_1)-s)} \right| \\
        &= - \log{ \left( \min_{1 \leq i,j \leq n} (r_i, c_j) \right)} + \frac{\norm{\vC}_{\max} \max(\norm{\vr}_1, \norm{\vc}_1)}{\gamma (\max(\norm{\vr}_1, \norm{\vc}_1)-s)}.
    \end{align*}
    Combining this with the bounds for $w^\ast$, we have 
    \begin{align*}
        \norm{(\vu^\ast,\vv^\ast, w^\ast)}_\infty \leq \frac{\norm{\vC}_{\max} \max(\norm{\vr}_1, \norm{\vc}_1)}{\gamma (\max(\norm{\vr}_1, \norm{\vc}_1)-s)} - \log{ \left( \min_{1 \leq i,j \leq n} (r_i, c_j) \right)}.
    \end{align*}
\end{proof}
\begin{corollary} \label{Bounds_for_(u,v)}
    The upper bound for $\norm{(\vy^\ast,\vz^\ast,t^\ast)}_\infty$ is
    \begin{equation} \label{infty_bounds}
        \frac{\norm{\vC}_{\max} \max\left\{\norm{\vr}_1, \norm{\vc}_1\right\}}{\max\left\{\norm{\vr}_1, \norm{\vc}_1\right\}-s} - \gamma \log{ \left( \min_{1 \leq i,j \leq n} (r_i, c_j) \right)} + \gamma.
    \end{equation}
\end{corollary}
\begin{proof}
    We have
    \begin{align*} 
        \norm{(\vy^\ast,\vz^\ast,t^\ast)}_\infty &=  \gamma \norm{(\vu^\ast + \ones,\vv^\ast + \ones, w^\ast - 1)}_\infty \\
        &\leq \gamma \left( \norm{(\vu^\ast,\vv^\ast, w^\ast)}_\infty + 1 \right) \\
        &\leq \frac{ \norm{\vC}_{\max} \max(\norm{\vr}_1, \norm{\vc}_1)}{ \max(\norm{\vr}_1, \norm{\vc}_1)-s} -\gamma \log{ \left( \min_{1 \leq i,j \leq n} (r_i, c_j) \right)} + \gamma, 
    \end{align*}
    where the last inequality comes as a result of Step 2.
\end{proof}
Note that $\norm{(\vy^\ast,\vz^\ast,t^\ast)}_2 \leq \sqrt{n} \norm{(\vy^\ast,\vz^\ast,t^\ast)}_\infty$, so we obtain $\Bar{R} = \mathcal{\widetilde{O}}\left(\sqrt{n} \norm{\vC}_{\max} \right)$.
\subsubsection{Step 3: Derive the Final APDAGD Computational Complexity (Theorem \ref{APDAGD_complexity})}
\label{step 3}
\begin{proof} [\unskip\nopunct]
    With the introduction of $\widetilde{\vr}$ and $\widetilde{\vc}$ in Algorithm \ref{alg:ApproxOT_APDAGD}, our POT constraints now becomes $
    \vA \vx = \widetilde{\vb}$, where $\widetilde{\vb} = (\widetilde{\vr}^\top, \widetilde{\vc}^\top, s)^\top$. We can still ensure that $\| \widetilde{\vr} \|_1 \geq s$ as follow 
    \begin{align*}
        \| \widetilde{\vr} \|_1 &= \left(1 - \dfrac{\widetilde{\varepsilon}}{8}\right)\|\vr\|_1 + \dfrac{\widetilde{\varepsilon}}{8} \\
        &= \dfrac{\widetilde{\varepsilon}}{8}\left(1 - \|\vr\|_1 \right) + \|\vr\|_1.
    \end{align*}
    If $\|\vr\|_1 \leq 1$, we obtain $\| \widetilde{\vr} \|_1 \geq \|\vr\|_1 \geq s$. And if $\|\vr\|_1 > 1$, note that as $\widetilde{\varepsilon} \leq \dfrac{8 (\|\vr\|_1-s)}{\|\vr\|_1-1}$, we have $\| \widetilde{\vr} \|_1 \geq s - \|\vr\|_1 + \|\vr\|_1 = s$. Similarly, we can ensure $\| \widetilde{\vc} \|_1 \geq s$. With these changes to $\widetilde{\vr}$ and $\widetilde{\vc}$, we obtain from Corollary \ref{Bounds_for_(u,v)}
    \begin{align}
        \norm{(\vy^\ast,\vz^\ast,t^\ast)}_\infty \leq \frac{ \norm{\vC}_{\max} \max(\norm{\widetilde{\vr}}_1, \norm{\widetilde{\vc}}_1)}{ \max(\norm{\widetilde{\vr}}_1, \norm{\widetilde{\vc}}_1)-s} -\gamma \log{ \left( \min_{1 \leq i,j \leq n} (\widetilde{r}_i, \widetilde{c}_j) \right)} + \gamma,
        \label{norm_infinity_bounds}
    \end{align}
    implying $\norm{(\vy^\ast,\vz^\ast,t^\ast)}_\infty = \mathcal{\widetilde{O}}\left( \norm{\vC}_{\max} \right)$ since $ \log{ \left( \min_{1 \leq i,j \leq n} ( \widetilde{r}_i, \widetilde{c}_j) \right)} = \widetilde{O}(\log(n/\varepsilon))$.

    From Proposition 4.10 from \citep{lin2019efficient}, the number of iterations for APDAGD Algorithm required to reach the tolerance $\varepsilon$ is
    \begin{equation} \label{iteration_bound}
        t \leq \max \Biggl\{ \mathcal{O}\left( \min \Biggl\{ \frac{n^{1/4}\sqrt{\Bar{R}\norm{C}_{max} \log{n}}}{\varepsilon}, \frac{\Bar{R}\norm{C}_{max} \log{n}}{\varepsilon^2} \Biggl\} \right), \mathcal{O} \left( \frac{\Bar{R}\log{n}}{\varepsilon} \right) \Biggl\}
    \end{equation}
    From the bound \eqref{norm_infinity_bounds} and the fact that $\norm{(\vy^\ast,\vz^\ast,t^\ast)}_2 \leq \sqrt{n} \norm{(\vy^\ast,\vz^\ast,t^\ast)}_\infty$, we obtain $\Bar{R} = \mathcal{\widetilde{O}}\left(\sqrt{n} \norm{\vC}_{\max} \right)$. Therefore, from \eqref{iteration_bound}, the total iteration complexity is $\mathcal{O}\left(n^{1/2}\norm{\vC}_{\max} \sqrt{\log{n}} / \varepsilon \right)$. Since each iteration of APDAGD algorithm requires $\mathcal{O}(n^2)$ arithmetic operations, and the rounding algorithm also requires $\mathcal{O}(n^2)$ arithmetic operations, we have the total number of arithmetic operations is the desired $\mathcal{O}\left(n^{5/2}\norm{\vC}_{\max} \sqrt{\log{n}} / \varepsilon \right)$. 
\end{proof}

\section{Dual Extrapolation}
\subsection{Second Order Characterization of Area Convexity}
In our context, the second order characterization of area-convexity, which was proven in \citep[Theorem 1.6]{Sherman-2017-Area} and restated in \citep[Theorem 3.2]{Jambulapati-2019-Direct}, is quite useful
\begin{theorem} \label{theorem:2nd_order_area_convex}
    For bilinear minimax objective and twice-differentiable regularizer r, if for all $\vz$ in the domain
    \begin{equation*}
        \left(
        \begin{array}{cc}
        \kappa \nabla^2 r(\vz) & -J \\
        J & \kappa \nabla^2 r(\vz)
        \end{array} \right) \succeq 0,
    \end{equation*}
    where $J$ is the Jacobian of the gradient operator g, then r is $3 \kappa$-area-convex with respect to g. 
\end{theorem} 
\subsection{$\ell1$ Penalization}

\begin{lemma}
    \label{lem:l1_regression}
    The $\ell_1$ penalized POT objective,
    \begin{equation} \label{l1_penalized}
        \min_{\vx \in \Delta_{n^2+2n}} \vd^\top \vx + 23 \norm{\vd}_\infty \norm{\vA \vx - \vb}_1,
    \end{equation}
    has an equal optimal value to that of the POT formulation \eqref{POT_areaconvex}. Moreover, any $\varepsilon$-approximate minimizer $\widetilde{\vx}$ of the $\ell_1$ penalized objective \eqref{l1_penalized} yields in $\mathcal{O}(n^2)$ time an $\varepsilon$-approximate transport plan $\bar{\vx}$ for the POT formulation \eqref{POT_areaconvex} (after applied the Rounding Algorithm \ref{alg:rounding}).
\end{lemma}
\begin{proof}
    Note that, intuitively, multiplicative constant $23$ comes from the guarantees of  $\textsc{Round-POT}$ (Theorem \ref{prop:rounding}). 
    
    Let $\tilde{\vx} \in$ $\argmin_{\vx \in \Delta_{n^2+2n}} \vd^\top \vx + 23 \norm{\vd}_\infty \norm{\vA \vx - \vb}_1$. Following the approach from \citep[Lemma 1]{Jambulapati-2019-Direct}, we claim there is some $\tilde{\vx}$ satisfying $\vA\tilde{\vx} = \vb$. Suppose otherwise, let $\norm{\vA\tilde{\vx} - \vb}_1 = \delta > 0$. Then, let $\bar{\vx}$ be the output after rounding $\tilde{\vx}$ with Algorithm \ref{alg:rounding}, so that $\vA\bar{\vx} = \vb, \norm{\tilde{\vx} - \bar{\vx}}_1 \leq 23\delta$ (Theorem \ref{prop:rounding}). Consider
	\begin{equation*}
	\vd^\top \bar{\vx} + 23 \norm{\vd}_\infty \norm{\vA\bar{\vx} - \vb}_1 = \vd^\top(\bar{\vx} - \tilde{\vx}) + \vd^\top \tilde{\vx} \leq \norm{\vd}_\infty \norm{\bar{\vx} - \tilde{\vx}}_1 + \vd^\top \tilde{\vx} \leq 23\norm{\vd}_\infty\delta + \vd^\top \tilde{\vx}.
	\end{equation*}
	The $\ell_1$ penalized objective value of $\bar{\vx}$ is smaller than that of $\tilde{\vx}$, hence this implies contradiction. From this claim, it is evident that 
	\begin{equation*}
	    \min_{\vx \in \Delta_{n^2+2n}} \vd^\top \vx + 23 \norm{\vd}_\infty \norm{\vA \vx - \vb}_1 = \vd^\top \tilde{\vx} + 23 \norm{\vd}_\infty \norm{\vA \tilde{\vx} - \vb}_1 = \vd^\top \tilde{\vx}. 
	\end{equation*}
	Let $\hat{\vx} \in$ $\argmin_{\vA \vx = \vb} \vd^\top \vx$. If we consider the objective \eqref{l1_penalized} for $\hat{\vx}$, we have
	\begin{align*}
	     \min_{\vA \vx = \vb} \vd^\top \vx &= \vd^\top \hat{\vx}\\ 
	     &= \vd^\top \hat{\vx} + 23 \norm{\vd}_\infty \norm{\vA \hat{\vx} - \vb}_1 \\
	     &\geq \min_{\vx \in \Delta_{n^2+2n}} \vd^\top \vx + 23 \norm{\vd}_\infty \norm{\vA \vx - \vb}_1 \\
	     &= \vd^\top \tilde{\vx}.
	\end{align*}
	Enforcing equality sign, we obtain the first claim that problems (\ref{POT_areaconvex}) and (\ref{l1_penalized}) have the same optimal value. Therefore, we can take any approximate minimizer to \eqref{l1_penalized} and round it to a transport plan $\bar{\vx}$ without increasing the POT objective.
\end{proof}
\subsection{Proof of Lemma \ref{9_area_convex}}
\begin{proof} [\unskip\nopunct]
    With the second order condition for area convexity (Theorem \ref{theorem:2nd_order_area_convex}) it suffices to show that for any z in the domain
    \begin{equation*}
        \xi = \left(
        \begin{array}{cc}
        3 \nabla^2 r(\vz) & -J \\
        J & 3 \nabla^2 r(\vz)
        \end{array} \right) \succeq 0,
    \end{equation*}
    where $J$ is the Jacobian of the operator $g$. 
    
    First, note that we can scale down both $J$ and $\nabla^2 r(\vz)$ with a common factor of $2 \norm{\vd}_\infty$ and still maintain the positive-semidefiniteness of the matrix $\xi.$ By straightforward differentiation, we obtain $J = \vA^\top$ and 
    \begin{equation*}
        \nabla^2 r(z) = \left(
        \begin{array}{cc}
        10 \: \diag(\frac{1}{x_j}) & 2 \: \vA^\top \diag(y_i)\\
        2 \: \diag(y_i) \vA & 2 \: \diag(\vA_i^\top \vx)
        \end{array} \right). 
    \end{equation*}
    Hence we get
    \begin{equation*}
        \xi = \left(
        \begin{array}{cccc}
        30 \: \diag(\frac{1}{x_j}) & 6 \: \vA^\top \diag(y_i) & \zeros & -\vA^\top \\
        6 \: \diag(y_i) \vA & 6 \: \diag(\vA_i^\top \vx) & \vA & \zeros \\ 
        \zeros & \vA^\top & 30 \: \diag(\frac{1}{x_j}) & 6 \: \vA^\top \diag(y_i) \\
        - \vA & \zeros &  6 \: \diag(y_i) \vA & 6 \: \diag(\vA_i^\top \vx)
        \end{array} \right). 
    \end{equation*}
    As the linear operator $\vA$ has the structure \eqref{matrix_A}, we obtain $\norm{\vA_{:j}}_1 \leq 3$.
    Given any arbitrary vector $(\va \: \vb \: \vc \: \vd)^\top,$ since $x_j \geq 0 \: \forall \: j$, we note that 
    \begin{align*}
        \va \xi \va^\top &= \sum_{j} \frac{30 a_j^2}{x_j} \geq \sum_{j} \norm{\vA_{:j}}_1  \frac{10 a_j^2}{x_j} = \sum_{i,j} A_{i,j} \frac{10 a_j^2}{x_j}, \\
        \vc \xi \vc^\top &= \sum_{j} \frac{30 c_j^2}{x_j} \geq \sum_{j} \norm{\vA_{:j}}_1  \frac{10 c_j^2}{x_j} = \sum_{i,j} A_{i,j} \frac{10 c_j^2}{x_j}.
    \end{align*}
    Therefore expand the terms and simplify $(\va \: \vb \: \vc \: \vd) \: \xi \: (\va \: \vb \: \vc \: \vd)^\top$ as follows
    \begin{align*}
        &(\va \: \vb \: \vc \: \vd) \: \xi \: (\va \: \vb \: \vc \: \vd)^\top \\ 
        &\geq \sum_{i,j} A_{ij} \left(\frac{10a_j^2}{x_j} + 6b_i^2x_j + \frac{10c_j^2}{x_j} + 6 d_i^2x_j + 12 a_j b_i y_i + 12 c_j d_i y_i - 2a_j d_i + 2c_j b_i \right).
    \end{align*}
    Note that as $y_i \in [-1,1],$ we have
    \begin{equation*}
        \frac{9a_j^2}{x_j} + 12 a_j b_i y_i + 4 b_i^2x_j \geq 0,\: \frac{9c_j^2}{x_j} + 12 c_j d_i y_i + 4 d_i^2x_j \geq 0.
    \end{equation*}
    So we obtain
    \begin{align*}
        &(\va \: \vb \: \vc \: \vd) \: \xi \: (\va \: \vb \: \vc \: \vd)^\top \\
        &\geq \sum_{i,j} A_{ij} \left(\frac{a_j^2}{x_j} + 2 b_i^2x_j + \frac{c_j^2}{x_j} + 2 d_i^2x_j - 2a_j d_i + 2c_j b_i \right) \\
        &= \sum_{i,j} A_{ij} \left( \left(\frac{a_j^2}{x_j} + d_i^2x_j - 2a_j d_i \right) + \left(b_i^2x_j + \frac{c_j^2}{x_j} + 2c_j b_i\right) + b_i^2x_j + d_i^2x_j \right) \\
        &\geq 0,
    \end{align*}
    since $x_j \geq 0$. Hence $\xi \succeq 0$ as desired. 
\end{proof} 

\subsection{Proof of Lemma \ref{T}}
By \citep[Lemma 2]{Jambulapati-2019-Direct}, we can evaluate the proposed algorithm with output $(\bar{\vx}, \bar{\vy}) \in \mathcal{X} \times \mathcal{Y}$ by the duality gap 
\begin{equation} \label{duality_gap}
    \max_{\vy \in \mathcal{Y}} F(\bar{\vx},\vy) - \min_{\vx \in \mathcal{X}} F(\vx, \bar{\vy}) \leq \varepsilon. 
\end{equation}
As we want to obtain the convergence guarantees in terms of duality gap \eqref{duality_gap}, we consider the following lemma which has been stated and proven in \citep[Corollary 1]{Jambulapati-2019-Direct}:
\begin{lemma} \label{regret_bound}
    Suppose that the proximal steps in General Dual Extrapolation Algorithm \ref{alg:General_Dual_Extrapolation} have an additive error $\varepsilon'$ during implementation and the regularizer $r$ is $\kappa$-area-convex with respect to the gradient operator $g$. Suppose for some $\vu \in \mathcal{Z}, \Theta \geq r(\vu) - r(\Bar{\vz})$. Then the output $\Bar{\vw}$ of Algorithm 
    \ref{alg:General_Dual_Extrapolation} satisfies
    \begin{equation*}
        \inner{g(\Bar{\vw})}{\Bar{\vw}-\vu} \leq \frac{2 \kappa \Theta}{T} + \varepsilon'.
    \end{equation*}
\end{lemma}
From this lemma, with a chosen additive error $\varepsilon' = \dfrac{\varepsilon}{2}$, if the aim is $\inner{g(\Bar{\vw})}{\Bar{\vw}-\vu} \leq \varepsilon$, then we want to have $\dfrac{2 \kappa \Theta}{T} \leq \dfrac{\varepsilon}{2}$. So $T$ is required to be $\left\lceil \dfrac{4 \kappa \Theta}{\varepsilon} \right\rceil$. Next, we look for a suitable $\Theta$.
\begin{lemma} \label{T}
    Lemma \ref{regret_bound} will apply to $\forall \: \vu \in \mathcal{Z}$ if $\Theta = 60 \norm{\vd}_\infty \log (n) + 6 \norm{\vd}_\infty$.
\end{lemma}
\begin{proof}
    Note that it is obvious that we want to choose
    \begin{equation*}
        \Theta \geq \sup_{\vz \in \mathcal{Z}} r(\vz) - \inf_{\vz \in \mathcal{Z}} r(\vz).
    \end{equation*}
    Denote the negative entropy as $h(\vx)$, it is well-known that
    \begin{align*}
        \max_{\vx \in \Delta_{d}} h(\vx) - \min_{\vx \in \Delta_{d}} h(\vx) = \log d.
    \end{align*}
    Next, we consider
    \begin{align*}
        \vx^\top \vA^\top (\vy^2) &= \inner{\vx}{\vA^\top (\vy^2)} \\
        &\leq \norm{\vx}_1 \norm{\vA^\top(\vy^2)}_\infty \\
        &= \norm{\vA^\top(\vy^2)}_\infty\\
        &\leq \norm{\vA^\top \ones_{2n+1}}_\infty \\
        &= \norm{\vA^\top}_\infty \\
        &= 3.
    \end{align*}
    On the other hand, it is obvious that $\vx^\top \vA^\top (\vy^2) \geq 0$. Combining the previous few inequalities, we have the following bound for the range of regularizer $r(\vx, \vy)$
    \begin{equation*}
        \sup_{\vz \in \mathcal{Z}} r(\vz) - \inf_{\vz \in \mathcal{Z}} r(\vz) \leq 2 \norm{\vd}_\infty \left( 10\log (n^2+2n) + 3 \right) = 20 \norm{\vd}_\infty \log (n^2+2n) + 6 \norm{\vd}_\infty.
    \end{equation*}
    It is noteworthy that with $n \geq 2, n^{3} = n \cdot n^2 \geq n \cdot (n+2) = n^2 + 2n$. Hence, we choose 
    \begin{equation*}
        \Theta = 60 \norm{\vd}_\infty \log (n) + 6 \norm{\vd}_\infty \geq 20 \norm{\vd}_\infty \log (n^2+2n) + 6 \norm{\vd}_\infty \geq \sup_{\vz \in \mathcal{Z}} r(\vz) - \inf_{\vz \in \mathcal{Z}} r(\vz).
    \end{equation*}
\end{proof}
\subsection{Proof of Theorem \ref{Complexity_for_AM}}
\begin{proof} [\unskip\nopunct]
    By \citep[Lemma 7, 8]{Jambulapati-2019-Direct}, the suboptimality gap can be reduced by a factor of $1/24$. There are two proximal step, and the second one deals with $\vs^t + \frac{1}{\kappa} g(\vz^t)$ which is bigger than $\vs^t$ in the first step. Hence, we endeavor to bound the number of iterations for the second bigger proximal step of $\vs^t + \frac{1}{\kappa} g(\vz^t)$. We have the proximal objective
    \begin{align*}
        prox^r_{\Bar{\vz}}(\vs^t + \frac{1}{\kappa} g(\vz^t)) &= \underset{\vz \in \mathcal{Z}} \argmin \inner{\vs^t + \frac{1}{\kappa} g(\vz^t)}{\vz} + V_{\Bar{\vz}}^r (\vz) \\
        &= \underset{\vz \in \mathcal{Z}} \argmin \inner{\vs^t + \frac{1}{\kappa} g(\vz^t)}{\vz} + r(\vz) - r(\Bar{\vz}) - \inner{\nabla r(\Bar{\vz})}{\vz - \Bar{\vz}} \\
        &= \underset{\vz \in \mathcal{Z}} \argmin \inner{\vs^t + \frac{1}{\kappa} g(\vz^t)-\nabla r(\Bar{\vz})}{\vz} + r(\vz).
    \end{align*}
    Let $\vz = (\vx,\vy)$, $\bar{\vz} = (\bar{\vx},\bar{\vy})$, $g(\vz) = (g_\vx(\vz),g_\vy(\vz))$, and $\vs^t = (\vs_\vx^t,\vs_\vy^t)$, we obtain
    \begin{align*}
        prox^r_{\Bar{\vz}}\left(\vs^t + \frac{1}{\kappa} g(\vz^t) \right) = &\underset{\vx \in \mathcal{X}, \vy \in \mathcal{Y}} \argmin \inner{\vs^t_\vx -\nabla_\vx r(\bar{\vx}, \bar{\vy}) + \frac{1}{\kappa} g_\vx(\vz^t)}{\vx} + \\
        &\inner{\vs^t_\vy -\nabla_\vy r(\bar{\vx}, \bar{\vy}) + \frac{1}{\kappa} g_\vy(\vz^t)}{\vy} + r(\vx,\vy).
    \end{align*}
    Note that $\nabla_\vy r(\bar{\vx}, \bar{\vy}) = \zeros$. Let $\vx^\ast$ and $\vy^\ast$ be the minimizer of the proximal objective, while $\vx_0$ and $\vy_0$ are the initializations, the suboptimality gap is 
    \begin{equation*}
        \delta = \inner{\vs^t_\vx -\nabla_\vx r(\bar{\vx}, \bar{\vy})+ \frac{1}{\kappa} g_\vx(\vz^t)}{\vx_0-\vx^\ast} + \inner{\vs^t_\vy + \frac{1}{\kappa} g_\vy(\vz^t)}{\vy_0-\vy^\ast} + r(\vx_0,\vy_0) - r(\vx^\ast,\vy^\ast).
    \end{equation*}
    Now the goal is to bound this suboptimality gap. By Lemma \ref{T}, as proximal steps are implemented with $\varepsilon/2$ accuracy, we get that $T = \lceil 36 \Theta / \varepsilon \rceil$ iterations would suffice. We can bound $\vs_\vx^t$ and $\vs_\vy^t$ as follows
    \begin{align*}
        \norm{\vs_\vx^t}_\infty &\leq T \frac{1}{2\kappa}  \norm{g_\vx(\vz)}_\infty \leq \frac{2 \Theta}{\varepsilon} \norm{\vd + 23 \norm{\vd}_\infty \vA^\top \vy}_\infty \leq \frac{2 \Theta}{\varepsilon} \left(\norm{\vd}_\infty + 23 \norm{\vd}_\infty \norm{\vA^\top \vy}_\infty \right) \\
        &\leq \frac{140 \Theta \norm{\vd}_\infty}{\varepsilon}, \\
        \norm{\vs_\vy^t}_1 &\leq T \frac{1}{2\kappa} \norm{g_\vy(\vz)}_1 \leq \frac{2 \Theta}{\varepsilon} \norm{-23 \norm{\vd}_\infty (\vA \vx -\vb)}_1 \leq \frac{46 \Theta \norm{\vd}_\infty}{\varepsilon} (\norm{\vA \vx}_1 + \norm{\vb}_1) \\
        &\leq \frac{276 \Theta \norm{\vd}_\infty}{\varepsilon}.
    \end{align*}
    Similarly, we have 
    \begin{align*}
        &\frac{1}{\kappa}  \norm{g_\vx(\vz)}_\infty \leq \frac{1 }{9} \norm{\vd + 23 \norm{\vd}_\infty \vA^\top \vy}_\infty \leq \frac{1}{9} \left(\norm{\vd}_\infty + 23 \norm{\vd}_\infty \norm{\vA^\top \vy}_\infty \right) \leq \frac{70 \norm{\vd}_\infty}{9}, \\
        &\frac{1}{\kappa} \norm{g_\vy(\vz)}_1 \leq \frac{1}{9} \norm{-23 \norm{\vd}_\infty (\vA \vx -\vb)}_1 \leq \frac{23 \norm{\vd}_\infty}{9} (\norm{\vA \vx}_1 + \norm{\vb}_1) \leq \frac{46 \norm{\vd}_\infty}{3}.
    \end{align*}
    Also note that 
    \begin{align*}
        \norm{\nabla_\vx r(\bar{\vx}, \bar{\vy})}_\infty = \norm{20 \norm{
        \vd}_\infty (1 - \log(n^2+2n)) \ones_{n^2 +2n}}_\infty \leq 20 \norm{
        \vd}_\infty (1 + 3\log(n))
    \end{align*}
    Hence, we can now proceed to bound the suboptimality gap
    \begin{align*}
        \delta &\leq \norm{\vs^t_\vx -\nabla_\vx r(\bar{\vx}, \bar{\vy}) + \frac{1}{\kappa} g_\vx(\vz^t)}_\infty \norm{\vx_0-\vx^\ast}_1 + \norm{\vs^t_\vy + \frac{1}{\kappa} g_\vy(\vz^t)}_1 \norm{\vy_0-\vy^\ast}_\infty + \Theta \\
        &\leq \left(\frac{140 \Theta \norm{\vd}_\infty}{\varepsilon} + 20 \norm{\vd}_\infty (1 + 3\log(n)) + \frac{70 \norm{\vd}_\infty}{9} + \frac{70 \Theta \norm{\vd}_\infty}{\varepsilon} + \frac{46 \norm{\vd}_\infty}{3} \right) \times 2 + \Theta \\
        &= \left( \frac{420 \norm{\vd}_\infty}{\varepsilon} + 1 \right) \Theta + 40 \norm{\vd}_\infty + 120\norm{\vd}_\infty \log(n) + \frac{416 \norm{\vd}_\infty}{9} \\
        &= \left( \frac{420 \norm{\vd}_\infty}{\varepsilon} + 3 \right) \Theta + \frac{668 \norm{\vd}_\infty}{9}. 
    \end{align*}
    Therefore, the number of iterations to obtain the desired output with $\varepsilon/2$-accuracy is 
    \begin{equation*}
        M = 24 \log_{24/23} \left( \frac{2 \delta}{\varepsilon} \right) \leq 24 \log \left( \left( \frac{840 \norm{\vd}_\infty}{\varepsilon^2} + \frac{6}{\varepsilon} \right) \Theta + \frac{1336 \norm{\vd}_\infty}{9} \right) = \mathcal{O} (\log \eta),
        \end{equation*}
    where $\eta = \log n \times \norm{\vd}_\infty \times \varepsilon^{-1}$. Each iteration can be done in $\mathcal{O} (n^2)$ time, so we obtain the desired complexity.
\end{proof}

\subsection{Proof of Theorem \ref{theorem:DE_complexity}}
\begin{proof}
        \label{proof:Complexity_for_AM}
        From Theorem \ref{Complexity_for_AM}, each proximal step in Algorithm \ref{alg:Dual_Extrapolation_POT} requires $\mathcal{O}(n^2 \log \eta)$ time where $\eta = \log n \norm{\vd}_\infty / \varepsilon$. Hence, with 2 proximal steps, Algorithm \ref{alg:Dual_Extrapolation_POT} can be done in the total of $\mathcal{O}(n^2 \log n \norm{\vd}_\infty \log \eta / \varepsilon) = \mathcal{\widetilde O}(n^2 \norm{\vd}_\infty / \varepsilon)$ time. This is $\mathcal{\widetilde O}(n^2 \norm{\vC}_{max} / \varepsilon)$ in term of the original cost matrix $\vC$.
\end{proof}

%% file: Appendix/experiment_setup_details.tex
\section{Further Experiment Setup Details}
In this section, we will provide the more detailed experimental setup for the several applications in the main text. 
\subsection{Computational Resources}

For all numerical experiments in this paper, we use Python 3.9 with all algorithms implemented using \texttt{numpy} version 1.23.2. We run all experiments on a Linux machine with an Intel Xeon Silver 4114 CPU with 40 threads and 128GB of RAM. No GPU is used.

\subsection{APDAGD vs DE}
Each RGB image of size $32 \times 32 \times 3$ 
is transformed into grayscale, downscaled to $10 \times 10$, and flattened into a histogram of $100$ bins, all using the \texttt{scikit-image} library. 
We add $10^{-6}$ to every bin to avoid zeros. For every chosen pair of marginals, we divide each histogram by the maximum mass of the two, giving one marginal with a total mass of $1$ and other possibly less than $1$. For the cost matrix $\vC$, we use the squared Euclidean distance between pixel locations and normalize so $\norm{\vC}_\text{max} = 1$. The total transported mass is $0.8$ times the minimum total mass between the marginals. 

\subsection{Color Transfer}
Each RGB image can be considered a point cloud of pixels, and applying the color palette of one image into another is equivalent to applying the OT transport map to convert one color distribution to another. We choose the source image (Figure \ref{fig:color_transfer_comparison}a, \url{https://flic.kr/p/Lm6gFA}) and the target image (Figure \ref{fig:color_transfer_comparison}b, \url{https://flic.kr/p/c89LBU}) images, taken from Flickr. These images are chosen because they have quite distinct color palettes and different sizes. We follow the setup by \citep{Blondel-2018-Smooth} and implement the experiments as follows. Each image of size $h \times w \times 3$ is considered a collection of $hw$ pixels in 3 dimensions. Each image is quantized using $k$-means with $n = 100$ centroids. We apply $k$-means clustering to find $n$ centroids and assign to each pixel the centroid it is closest to. The image now becomes a color histogram with $n$ bins representing the centroids. For a pair of source and target images, similar to the previous subsection, we divide each histogram by the maximum total mass, yielding $\max\{\norm{\vr}_1, \norm{\vc}_1\} = 1$. Let $\va_1, \ldots, \va_n$ and $\vb_1, \ldots, \vb_n$ be the collections of centroids for the source and target images. The cost matrix is defined as $C_{i, j} = \norm{\va_i - \vb_j}_2^2$. With a partial transport map $\vX \in \RR_{+}^{n \times n}$, every centroid $\va_i$ in the source image is transformed to 
$\hat{\va}_i = (\sum_{j=1}^{n} X_{i, j} \vb_j) / (\sum_{j=1}^{n} X_{i, j})$. All pixels in the source image previously assigned to $\va_i$ are now assigned to $\hat{\va}_i$. The total transported mass is set to $s = \alpha \min \{\norm{\vr}_1, \norm{\vc}_1\}$, where $\alpha \in [0, 1]$. To approximate the POT solution, we set $\varepsilon = 10^{-2}$ and run both Sinkhorn \cite{nhatho-mmpot} and APDAGD for 1,000 iterations. We set $\alpha = 0.1$, corresponding to transporting exactly 10\% of the allowed mass.
We plot the optimality gap $\inner{\vC}{\vX} - f^*$ for each transport map after every iteration in Figure \ref{fig:color_transfer_comparison} (c). As explained earlier, Sinkhorn with the accompanying rounding algorithm by \citep{altschuler2017near} does not produce a primal feasible solution, and its primal gap does not satisfy an error of $\varepsilon = 10^{-2}$ (indicated by the red line).

\subsection{Point Cloud Registration}
Previously, \cite{qin2022rigid} proposed a procedure to find such a transformation using partial optimal transport. Let the two marginal distributions be $\vr = \frac{1}{m} \ones_m$ and $\vc = \frac{1}{n} \ones_n$ and the cost matrix be $C_{i, j} = \norm{\vx_i - \vy_j}_2^2$. Given an optimal transport matrix $\vT \in \RR^{m \times n}$ between these marginals, the rotation matrix $\vR$ and translation vector $\vt$ are obtained by minimizing the energy: 
\begin{align*}
    \min_{\vR, \vt} ~ \sum_{i=1}^{m} \sum_{j=1}^{n} T_{i, j} \norm{\vx_i - (\vR \vy_j + \vt)}_2^2.
\end{align*}
This problem admits the closed-form solution
\begin{align*}
    \vR = \mathbf{V} \vS \mathbf{U}^\top, \quad \vt = \vu_x - \vR \vu_y,
\end{align*}
where $\vu_x = \frac{1}{m} \sum_{i=1}^{m} \vx_i$, $\vu_y = \frac{1}{n} \sum_{j=1}^{n} \vy_j$. The matrices $\mathbf{U}$, $\vS$ and $\mathbf{V}$ are obtained as follows. Let $\hat{\vX} \in \RR^{m \times 3}$ whose $i$th row is $(\vx_i - \vu_x)^\top$. Similar for $\hat{\vY} \in \RR^{n \times 3}$. Then, obtain the singular value decomposition of the matrix $\hat{\vX} \vT^\top \hat{\vX}^\top$ as $\mathbf{U} \mathbf{\Lambda} \mathbf{V}^\top$. Finally, $\vS = \diag(1, 1, \det(\mathbf{V} \mathbf{U}^\top))$.

To find the the matrix $\vT$, we follow the iterative procedure in \cite[Algorithm 1]{qin2022rigid} in which the point cloud $Q$ is updated gradually until convergence. In our implementation, we set the initial value for $\gamma$ (strength of entropic regularization) to 0.004 and the annealing rate of $\gamma$ to $\lambda = 0.99$. With each $\gamma$, we find the OT matrix using Sinkhorn, and the POT matrix using two methods: Sinkhorn and APDAGD. The iterative process ends when the change of $\vR$ in Frobenius norm falls below $10^{-6}$.

%% file: Appendix/extra_experiments.tex
\section{Further Numerical Experiments}

\subsection{Run Time for Varying $\varepsilon$}
\begin{figure}
    \centering
    \includegraphics[width=0.5\linewidth]{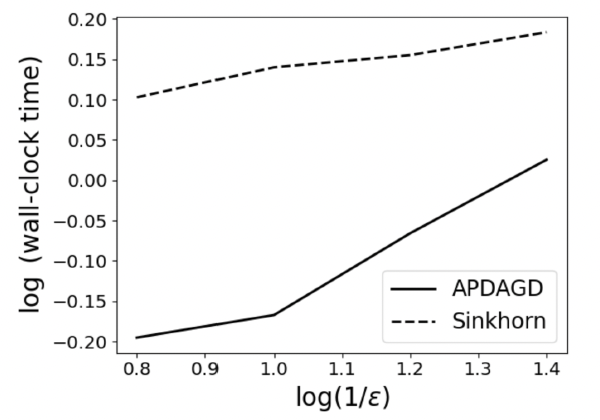}
    \caption{Comparison of wall-clock running time between APDAGD and Sinkhorn for varying $\varepsilon$}
    \label{fig:runtime_for_varying_epsilon}
\end{figure}
With the similar setting to Figure \ref{fig:de_vs_apdagd} of using images in the CIFAR-10 dataset, we provide additional experiments showcasing that APDAGD is better than Sinkhorn in wall-clock time in terms of both aggregate run time cost in Figure \ref{fig:runtime_for_varying_epsilon}. Moreover, for aggregate wall-clock time, we note that previous works showed that gradient methods and especially APDAGD are practically faster than Sinkhorn \cite[Figure 1]{Dvurechensky-2018-Computational}. This claim is further supported by Figure \ref{fig:runtime_for_varying_epsilon} since it reproduce the results from \cite[Figure 1]{Dvurechensky-2018-Computational} which compare the runtime of Sinkhorn and APDAGD in the context of OT. 
\subsection{Revised Sinkhorn}
We again utilize the same setting Figure \ref{fig:de_vs_apdagd} and Run Time for Varying $\varepsilon$ with CIFAR-10 dataset. On the left, we empirically verify our theoretical bounds in Theorem \ref{contraint_violation}, by showing that increasing $A$ potentially lead to a better feasibility of Sinkhorn. For the figure on the right, we show that the increase of $A$ to a sufficient size can lead to more number of iterations to convergence of Sinkhorn. This is also suggested by the worsened theoretical complexity of revised Sinkhorn (Theorem \ref{them:revised_sinkhorn_complexity}).
\begin{figure}
    \centering
    \includegraphics[width=0.8\linewidth]{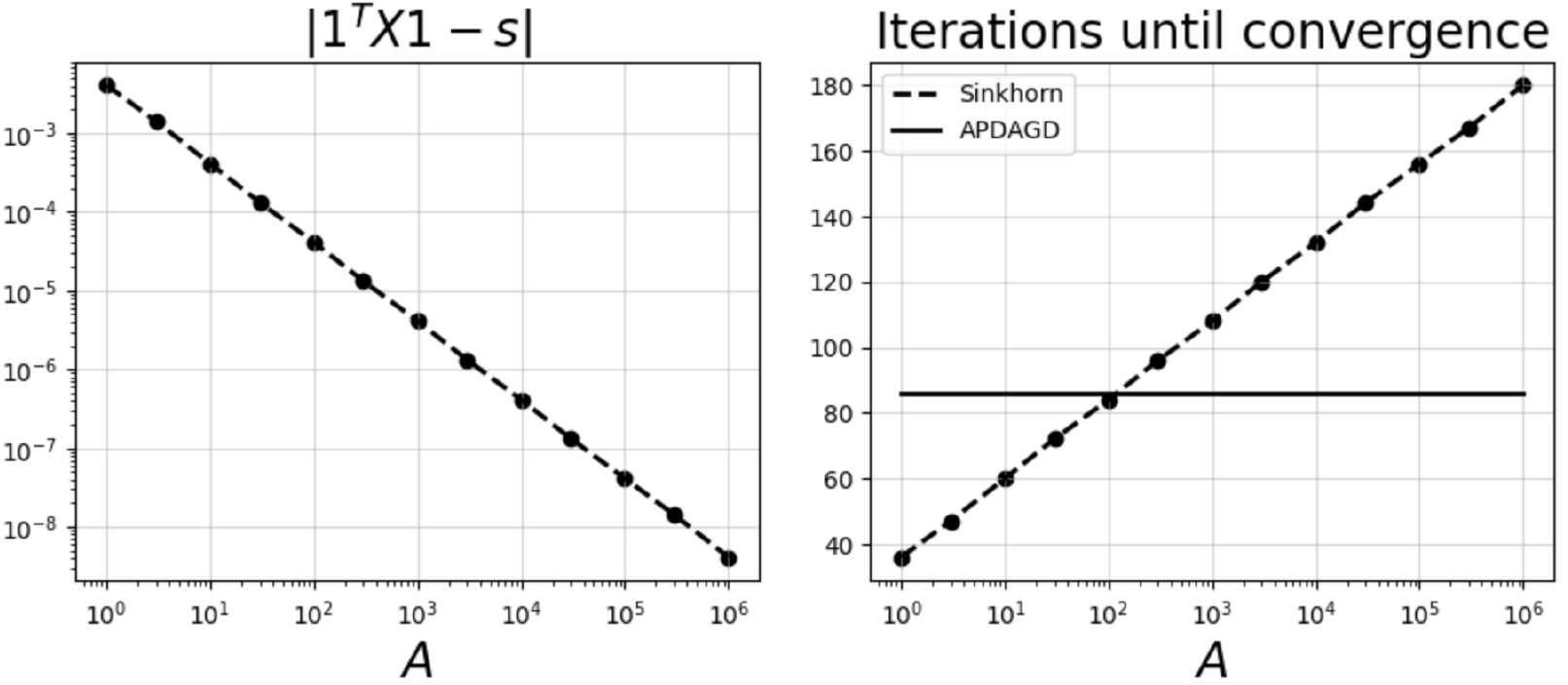}
    \caption{Left: Sinkhorn constraint violation with respect to $A$. Right: Number of iterations for both APDAGD and Sinkhorn until a pre-defined suboptimality is achieved.}
    \label{fig:revised_sinkhorn}
\end{figure}

\subsection{Domain Adaptation}
In our setting, the source and target datasets contain different numbers of data points ($N_s = 300$ and $N_t = 400$, respectively). For OT, we use Sinkhorn to approximate the OT matrix of size 300 $\times$ 400, then transform the source examples according to Courty et al. (2017). For POT, we use K-means clustering to transform the target domains to a histogram of $300$ bins. The two unbalanced marginals are then normalized so that $\left\| r \right\|_1 = 0.75, \left\| c \right\|_1 = 1.0$, and we set $s = 0.999 \times \min\{ \left\| r \right\|_1, \left\| c \right\|_1 \}$ then use APDAGD to find the POT matrix between these two unbalanced marginals. The source examples are transformed similarly. Finally, for both methods (OT and POT), we train a support vector machine with the radial basis kernel (parameter $\sigma^2 = 1$) using the transformed source domain. In the Figure \ref{fig:POT_vs_OT} we report the accuracy on the target domain using 20,000 test examples. All hyperparameters are set to be equal.

Here we implement and compare Domain Adaptation with OT to Domain Adaptation with POT, which is calculated with two different algorithms Sinkhorn (infeasible rounding) and APDAGD (with our novel \textsc{Round-POT}). In \citep{courty2017joint}, the authors presented an OT-based method to transform a source domain to a reference target domain such that the their densities are approximately equal. We present a binary classification setting involving the "moons" dataset. The source dataset is displayed as colored scatter points. The target dataset comes from the same distribution, but the points go to a rotation of 60 degrees (black scatter points), representing a domain covariate shift. Additional experimental details are further described in Further Experiments Setup Details section in Appendix. 
To summarize the results, POT offers much more practicality as it does not require normalizing both marginals to have a sum of 1 in contrast to OT. More importantly, POT avoids bad matchings to outliers due to the flexibility in choosing the amount of mass transported, resulting in a higher accuracy than OT. Furthermore, linking back to Remark \ref{remark:violation}, the infeasibility Sinkhorn degrades the practical performance of POT for Domain Adaption and consequently leads to a worse accuracy compared APDAGD. This is because the optimal mapping between the source and target domains requires the exact fraction of mass to be transported which APDAGD fulfills practically in contrast to Sinkhorn.  

\begin{figure}
    \centering
    \includegraphics[width=1\linewidth]{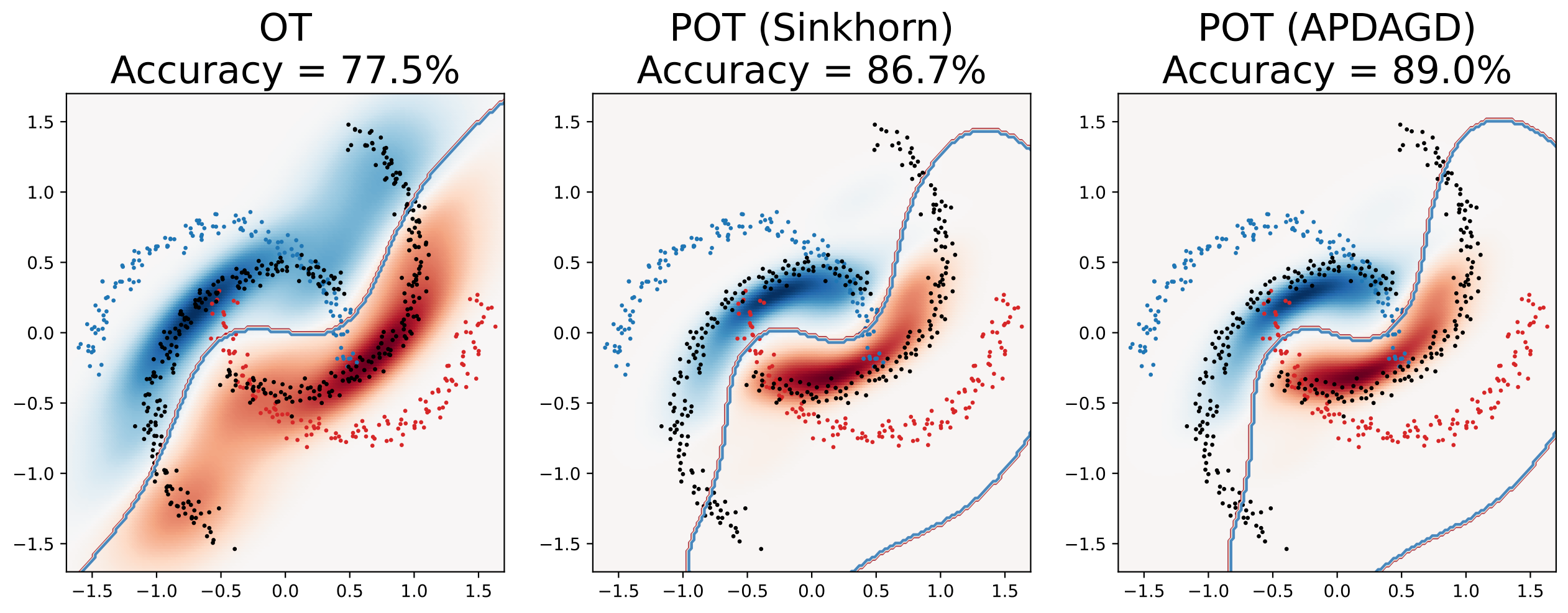}
    \caption{Domain adaptation with OT and with POT. POT offers a flexibility in choosing the mass transported and helps avoid the need to normalize both marginals to have a sum of $1$, resulting in better complexity than OT. APDAGD demonstrates better accuracy in the novel domain than Sinkhorn, which suffers from infeasibility.}
    \label{fig:POT_vs_OT}
\end{figure}

\subsection{Synthetic Data}
\begin{figure}[ht]
    \centering
    \includegraphics[width=1\linewidth]{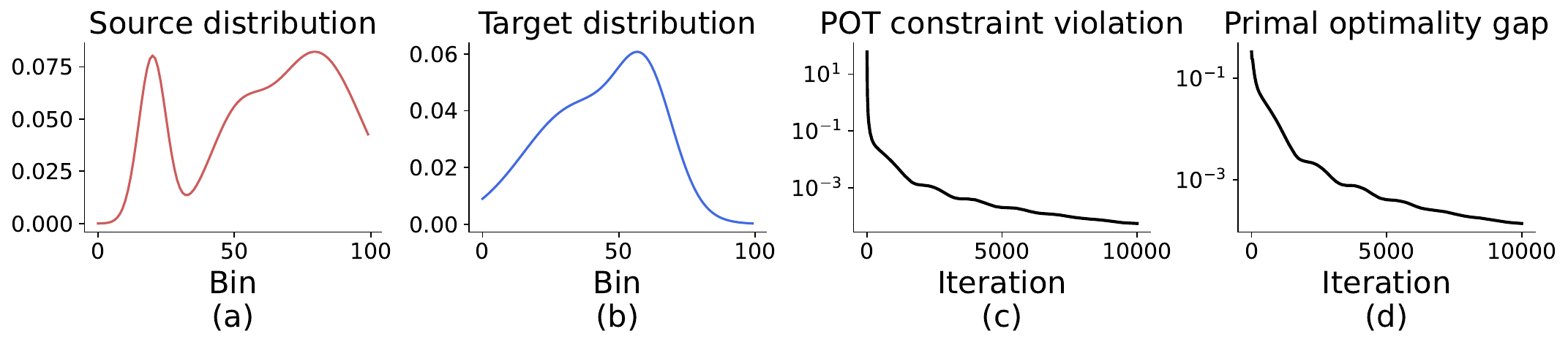}
    \caption{Convergence of APDAGD on a synthetic setting.}
    \label{fig:apdagd_synthetic}
\end{figure}

In this section we design a POT setting and evaluate the convergence of APDAGD. The two marginals are unnormalized Gaussian mixtures. We then discretize them into histograms of $n = 100$ bins each and perform normalization so that $\norm{\vr}_1 = 5$, $\norm{\vc}_1 = 3$. The total transported mass is set to $s = 0.9 \min\{ \norm{\vr}_1, \norm{\vc}_1 \} = 2.7$, and the cost matrix represents squared Euclidean distances between every pair of bins in the source and target histograms. In other words, $C_{i, j} = (i - j)^2$ for all $i, j = 1, \ldots, n$. Finally, we divide $\vC$ by its maximum entry to have $\norm{\vC}_\text{max} = 1$.

We use \texttt{cvxpy} with the GUROBI backend to solve the linear problem \eqref{prob:pot_with_pq}, and denote by $f^*$ the optimal objective value. This will be used to calculate the primal optimality gap.

The approximate POT solution is solved using APDAGD, where we set the additive suboptimality error $\varepsilon$ to $10^{-3}$. We run APDAGD for 10,000 iterations and keep track of two errors. The first error is constraint violation, equal to $\norm{\vX \ones + \vp - \vr}_1 + \norm{\vX^\top \ones + \vq - \vc}_1 + \lvert \ones^\top \vX \ones - s \rvert$ where $\vX$, $\vp$ and $\vq$ are part of the solution $\vx_k$ found after each iteration. The second error is the optimality gap for the primal objective. To calculate this, we need to round the solution so that it becomes primal feasible. We use our rounding algorithm described in Rounding Algorithm Section to round $\vX$, $\vp$, and $\vq$ after each iteration, giving us $\bar{\vX}$, $\bar{\vp}$ and $\bar{\vq}$. The primal gap then is $\inner{\vC}{\bar{\vX}} - f^*$.

\subsection{Scalability of APDAGD}
Given the complexity claims in the paper, it is important to show how running time scales with $n$, the number of supports in each marginal. In the below Figure \ref{fig:scalability}, we show the wall-clock time of running APDAGD to solve the POT problem between two Gaussian mixtures (shown in the end of the appendix). We fix the error tolerance to $\varepsilon = 10^{-1}$ and vary $n$ in the set $\{ 10, 30, 100, 300, 1000, 3000, 10000 \}$. The figure plots wall-clock time in seconds against $n$. The expected slope of the best-fit line between $\log(\text{time})$ and $\log(n)$ should be at most 2.5, as the complexity is $\widetilde{O}(n^{2.5} / \varepsilon)$. We find that the slope is $2.19$, which is consistent with the upper bound. Furthermore, the correlation is statistically significant.

\begin{figure}
    \centering
    \includegraphics[width=0.4\linewidth]{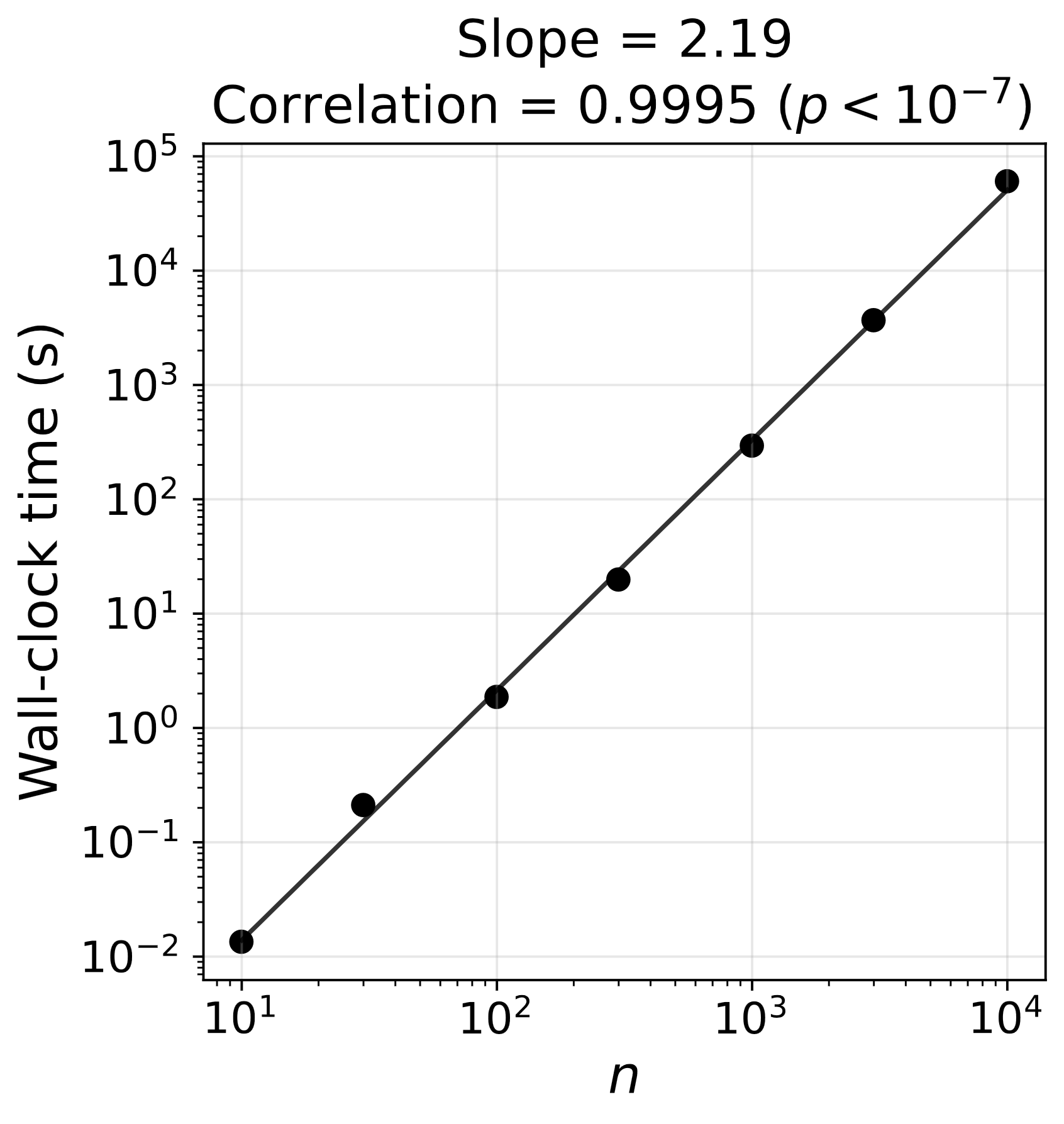}
    \caption{Wall-clock time of solving POT with APDAGD against $n$}
    \label{fig:scalability}
\end{figure}

\newpage

%% file: main.bbl
\begin{thebibliography}{48}
\providecommand{\natexlab}[1]{#1}

\bibitem[{Altschuler, Weed, and Rigollet(2017)}]{altschuler2017near}
Altschuler, J.; Weed, J.; and Rigollet, P. 2017.
\newblock Near-linear time approximation algorithms for optimal transport via Sinkhorn iteration.
\newblock In \emph{Advances in Neural Information Processing Systems}, 1964--1974.

\bibitem[{Balaji, Chellappa, and Feizi(2020)}]{balaji2020robust}
Balaji, Y.; Chellappa, R.; and Feizi, S. 2020.
\newblock Robust optimal transport with applications in generative modeling and domain adaptation.
\newblock \emph{Advances in Neural Information Processing Systems}, 33: 12934--12944.

\bibitem[{Benamou et~al.(2015)Benamou, Carlier, Cuturi, Nenna, and Peyr{\'e}}]{benamou2015iterative}
Benamou, J.-D.; Carlier, G.; Cuturi, M.; Nenna, L.; and Peyr{\'e}, G. 2015.
\newblock Iterative Bregman projections for regularized transportation problems.
\newblock \emph{SIAM Journal on Scientific Computing}, 37(2): A1111--A1138.

\bibitem[{Blondel, Seguy, and Rolet(2018{\natexlab{a}})}]{Blondel-2018-Smooth}
Blondel, M.; Seguy, V.; and Rolet, A. 2018{\natexlab{a}}.
\newblock Smooth and Sparse Optimal Transport.
\newblock In \emph{AISTATS}, 880--889.

\bibitem[{Blondel, Seguy, and Rolet(2018{\natexlab{b}})}]{blondel2018smooth}
Blondel, M.; Seguy, V.; and Rolet, A. 2018{\natexlab{b}}.
\newblock Smooth and sparse optimal transport.
\newblock In \emph{International conference on artificial intelligence and statistics}, 880--889. PMLR.

\bibitem[{Bonneel and Coeurjolly(2019)}]{Bonneel_2019}
Bonneel, N.; and Coeurjolly, D. 2019.
\newblock SPOT: Sliced Partial Optimal Transport.
\newblock \emph{ACM Transactions on Graphics}.

\bibitem[{Bonneel et~al.(2015)Bonneel, Rabin, Peyr{\'e}, and Pfister}]{bonneel2015sliced}
Bonneel, N.; Rabin, J.; Peyr{\'e}, G.; and Pfister, H. 2015.
\newblock Sliced and radon wasserstein barycenters of measures.
\newblock \emph{Journal of Mathematical Imaging and Vision}, 51(1): 22--45.

\bibitem[{Caffarelli and McCann(2010)}]{Caffarelli_2010}
Caffarelli, L.~A.; and McCann, R.~J. 2010.
\newblock Free boundaries in optimal transport and Monge-Ampère obstacle problems.
\newblock \emph{Annals of Mathematics}, 171(2): 673--730.

\bibitem[{Chapel, Alaya, and Gasso(2020)}]{Chapel-nips2020}
Chapel, L.; Alaya, M.~Z.; and Gasso, G. 2020.
\newblock Partial Optimal Tranport with applications on Positive-Unlabeled Learning.
\newblock In \emph{Advances in Neural Information Processing Systems 33}.

\bibitem[{Chizat et~al.(2015)Chizat, Peyré, Schmitzer, and Vialard}]{Chizat_2015}
Chizat, L.; Peyré, G.; Schmitzer, B.; and Vialard, F.-X. 2015.
\newblock Unbalanced Optimal Transport: Dynamic and Kantorovich Formulation.

\bibitem[{Chizat et~al.(2017)Chizat, Peyré, Schmitzer, and Vialard}]{chizat2017scaling}
Chizat, L.; Peyré, G.; Schmitzer, B.; and Vialard, F.-X. 2017.
\newblock Scaling Algorithms for Unbalanced Transport Problems.
\newblock arXiv:1607.05816.

\bibitem[{Courty et~al.(2017{\natexlab{a}})Courty, Flamary, Habrard, and Rakotomamonjy}]{courty2017joint}
Courty, N.; Flamary, R.; Habrard, A.; and Rakotomamonjy, A. 2017{\natexlab{a}}.
\newblock Joint distribution optimal transportation for domain adaptation.
\newblock \emph{Advances in neural information processing systems}, 30.

\bibitem[{Courty et~al.(2017{\natexlab{b}})Courty, Flamary, Tuia, and Rakotomamonjy}]{Courty-2017-Optimal}
Courty, N.; Flamary, R.; Tuia, D.; and Rakotomamonjy, A. 2017{\natexlab{b}}.
\newblock Optimal Transport for Domain Adaptation.
\newblock \emph{IEEE Transactions on Pattern Analysis and Machine Intelligence}, 39(9): 1853--1865.

\bibitem[{Cuturi(2013)}]{Cuturi-2013-Sinkhorn}
Cuturi, M. 2013.
\newblock Sinkhorn distances: {L}ightspeed computation of optimal transport.
\newblock In \emph{Advances in Neural Information Processing Systems}, 2292--2300.

\bibitem[{Cuturi and Peyr{\'e}(2016)}]{cuturi2016smoothed}
Cuturi, M.; and Peyr{\'e}, G. 2016.
\newblock A smoothed dual approach for variational Wasserstein problems.
\newblock \emph{SIAM Journal on Imaging Sciences}, 9(1): 320--343.

\bibitem[{Dennis and Mor{\'e}(1977)}]{dennis1977quasi}
Dennis, J.~E., Jr; and Mor{\'e}, J.~J. 1977.
\newblock Quasi-Newton methods, motivation and theory.
\newblock \emph{SIAM review}, 19(1): 46--89.

\bibitem[{Dvinskikh and Tiapkin(2021)}]{pmlr-v130-dvinskikh21a}
Dvinskikh, D.; and Tiapkin, D. 2021.
\newblock Improved Complexity Bounds in Wasserstein Barycenter Problem.
\newblock In Banerjee, A.; and Fukumizu, K., eds., \emph{Proceedings of The 24th International Conference on Artificial Intelligence and Statistics}, volume 130 of \emph{Proceedings of Machine Learning Research}, 1738--1746. PMLR.

\bibitem[{Dvurechensky, Gasnikov, and Kroshnin(2018)}]{Dvurechensky-2018-Computational}
Dvurechensky, P.; Gasnikov, A.; and Kroshnin, A. 2018.
\newblock Computational Optimal Transport: {C}omplexity by Accelerated Gradient Descent Is Better Than by {S}inkhorn’s Algorithm.
\newblock In \emph{International conference on machine learning}, 1367--1376.

\bibitem[{Figalli(2010)}]{figalli2010optimal}
Figalli, A. 2010.
\newblock The optimal partial transport problem.
\newblock \emph{Archive for rational mechanics and analysis}, 195(2): 533--560.

\bibitem[{Gramfort, Peyr{\'{e}}, and Cuturi(2015)}]{Gramfort_2015}
Gramfort, A.; Peyr{\'{e}}, G.; and Cuturi, M. 2015.
\newblock Fast Optimal Transport Averaging of Neuroimaging Data.
\newblock \emph{CoRR}, abs/1503.08596.

\bibitem[{Guminov et~al.(2021)Guminov, Dvurechensky, Tupitsa, and Gasnikov}]{guminov2021accelerated}
Guminov, S.; Dvurechensky, P.; Tupitsa, N.; and Gasnikov, A. 2021.
\newblock Accelerated Alternating Minimization, Accelerated Sinkhorn's Algorithm and Accelerated Iterative Bregman Projections.
\newblock arXiv:1906.03622.

\bibitem[{Jambulapati, Sidford, and Tian(2019)}]{Jambulapati-2019-Direct}
Jambulapati, A.; Sidford, A.; and Tian, K. 2019.
\newblock A Direct $\widetilde{O}(1/\varepsilon)$ Iteration Parallel Algorithm for Optimal Transport.
\newblock \emph{ArXiv Preprint: 1906.00618}.

\bibitem[{Janati, Cuturi, and Gramfort(2019)}]{Janati_Wasserstein_2019}
Janati, H.; Cuturi, M.; and Gramfort, A. 2019.
\newblock Wasserstein Regularization for Sparse Multi-task Regression.
\newblock In \emph{AISTATS}.

\bibitem[{Kantorovich(1942)}]{Kantorovich-1942-Translocation}
Kantorovich, L.~V. 1942.
\newblock On the translocation of masses.
\newblock In \emph{Dokl. Akad. Nauk. USSR (NS)}, volume~37, 199--201.

\bibitem[{Kawano, Koide, and Otaki(2021)}]{Kawano_2021}
Kawano, K.; Koide, S.; and Otaki, K. 2021.
\newblock Partial Wasserstein Covering.
\newblock \emph{CoRR}, abs/2106.00886.

\bibitem[{Krizhevsky and Hinton(2009)}]{krizhevsky2009learning}
Krizhevsky, A.; and Hinton, G. 2009.
\newblock Learning multiple layers of features from tiny images.
\newblock \emph{Technical Report TR-2009}.

\bibitem[{Lan, Ouyang, and Zhou(2021)}]{lan2021graph}
Lan, G.; Ouyang, Y.; and Zhou, Y. 2021.
\newblock Graph topology invariant gradient and sampling complexity for decentralized and stochastic optimization.
\newblock arXiv:2101.00143.

\bibitem[{Le et~al.(2021)Le, Nguyen, Pham, and Ho}]{nhatho-mmpot}
Le, K.; Nguyen, H.; Pham, T.; and Ho, N. 2021.
\newblock On Multimarginal Partial Optimal Transport: Equivalent Forms and Computational Complexity.

\bibitem[{Liero, Mielke, and Savar\'{e}(2018)}]{Liero_Optimal_2018}
Liero, M.; Mielke, A.; and Savar\'{e}, M.~I. 2018.
\newblock Optimal entropy-transport problemsand a new {H}ellinger–{K}antorovich distance between positive measures.
\newblock \emph{Inventiones Mathematicae}, 211: 969--1117.

\bibitem[{Lin, Ho, and Jordan(2019)}]{lin2019efficient}
Lin, T.; Ho, N.; and Jordan, M. 2019.
\newblock On Efficient Optimal Transport: An Analysis of Greedy and Accelerated Mirror Descent Algorithms.
\newblock In \emph{International Conference on Machine Learning}, 3982--3991.

\bibitem[{Liu et~al.(2020)Liu, Zhang, Xie, Shen, Qian, and Zheng}]{Liu_2020}
Liu, W.; Zhang, C.; Xie, J.; Shen, Z.; Qian, H.; and Zheng, N. 2020.
\newblock Partial Gromov-Wasserstein Learning for Partial Graph Matching.
\newblock \emph{CoRR}, abs/2012.01252.

\bibitem[{Nesterov(2005)}]{nesterov2005smooth}
Nesterov, Y. 2005.
\newblock Smooth minimization of non-smooth functions.
\newblock \emph{Mathematical programming}, 103(1): 127--152.

\bibitem[{Nesterov(2007)}]{nesterov_2006}
Nesterov, Y. 2007.
\newblock Dual extrapolation and its applications to solving variational inequalities and related problems.
\newblock \emph{Mathematical Programming}, 109(2-3): 319--344.

\bibitem[{Nguyen and Maguluri(2023)}]{hoang_nonlinear_sa}
Nguyen, H.~H.; and Maguluri, S.~T. 2023.
\newblock Stochastic Approximation for Nonlinear Discrete Stochastic Control: Finite-Sample Bounds.
\newblock arXiv:2304.11854.

\bibitem[{Nguyen et~al.(2022{\natexlab{a}})Nguyen, Nguyen, Pham, Ho et~al.}]{nguyen2022improving}
Nguyen, K.; Nguyen, D.; Pham, T.; Ho, N.; et~al. 2022{\natexlab{a}}.
\newblock Improving mini-batch optimal transport via partial transportation.
\newblock In \emph{International Conference on Machine Learning}, 16656--16690. PMLR.

\bibitem[{Nguyen et~al.(2022{\natexlab{b}})Nguyen, Nguyen, Zhou, and Nguyen}]{gem-uot}
Nguyen, Q.~M.; Nguyen, H.~H.; Zhou, Y.; and Nguyen, L.~M. 2022{\natexlab{b}}.
\newblock On Unbalanced Optimal Transport: Gradient Methods, Sparsity and Approximation Error.

\bibitem[{Nietert, Cummings, and Goldfeld(2023)}]{nietert2023robust}
Nietert, S.; Cummings, R.; and Goldfeld, Z. 2023.
\newblock Robust Estimation under the Wasserstein Distance.
\newblock \emph{arXiv preprint arXiv:2302.01237}.

\bibitem[{Pele and Werman(2008)}]{Pele2008ALT}
Pele, O.; and Werman, M. 2008.
\newblock A Linear Time Histogram Metric for Improved SIFT Matching.
\newblock In \emph{European Conference on Computer Vision}.

\bibitem[{Pham et~al.(2020)Pham, Le, Ho, Pham, and Bui}]{UOT_complexity}
Pham, K.; Le, K.; Ho, N.; Pham, T.; and Bui, H. 2020.
\newblock On Unbalanced Optimal Transport: An Analysis of Sinkhorn Algorithm.
\newblock In \emph{ICML}.

\bibitem[{Polyak(1969)}]{polyak1969conjugate}
Polyak, B.~T. 1969.
\newblock The conjugate gradient method in extremal problems.
\newblock \emph{USSR Computational Mathematics and Mathematical Physics}, 9(4): 94--112.

\bibitem[{Qin et~al.(2022)Qin, Zhang, Liu, and Chen}]{qin2022rigid}
Qin, H.; Zhang, Y.; Liu, Z.; and Chen, B. 2022.
\newblock Rigid Registration of Point Clouds Based on Partial Optimal Transport.
\newblock In \emph{Computer Graphics Forum}, volume~41, 365--378. Wiley Online Library.

\bibitem[{Rolet, Cuturi, and Peyré(2016)}]{pmlr-v51-rolet16}
Rolet, A.; Cuturi, M.; and Peyré, G. 2016.
\newblock Fast Dictionary Learning with a Smoothed Wasserstein Loss.
\newblock In Gretton, A.; and Robert, C.~C., eds., \emph{Proceedings of the 19th International Conference on Artificial Intelligence and Statistics}, volume~51 of \emph{Proceedings of Machine Learning Research}, 630--638. Cadiz, Spain: PMLR.

\bibitem[{Rubner, Tomasi, and Guibas(2000)}]{rubner2000earth}
Rubner, Y.; Tomasi, C.; and Guibas, L.~J. 2000.
\newblock The earth mover's distance as a metric for image retrieval.
\newblock \emph{International journal of computer vision}, 40(2): 99--121.

\bibitem[{Sarlin et~al.(2019)Sarlin, DeTone, Malisiewicz, and Rabinovich}]{Sarlin_2019}
Sarlin, P.; DeTone, D.; Malisiewicz, T.; and Rabinovich, A. 2019.
\newblock SuperGlue: Learning Feature Matching with Graph Neural Networks.
\newblock \emph{CoRR}, abs/1911.11763.

\bibitem[{Sherman(2017)}]{Sherman-2017-Area}
Sherman, J. 2017.
\newblock Area-convexity, $\ell_{\infty}$ regularization, and undirected multicommodity flow.
\newblock In \emph{STOC}, 452--460. ACM.

\bibitem[{Villani(2008)}]{Villani-09-Optimal}
Villani, C. 2008.
\newblock \emph{Optimal transport: Old and New}.
\newblock Springer.

\bibitem[{Wang et~al.(2022)Wang, Xue, Lei, and Xia}]{Wang_2022}
Wang, Z.; Xue, N.; Lei, L.; and Xia, G.-S. 2022.
\newblock Partial Wasserstein Adversarial Network for Non-rigid Point Set Registration.
\newblock In \emph{International Conference on Learning Representations}.

\bibitem[{Yang and Uhler(2019)}]{Yang_Scalable_2019}
Yang, K.~D.; and Uhler, C. 2019.
\newblock Scalable Unbalanced Optimal Transport using Generative Adversarial Networks.
\newblock In \emph{ICLR}.

\end{thebibliography}
